%% file: main.tex
\newif\ifarxiv
\arxivtrue

\newif\ifcolt

\ifarxiv
\documentclass[11pt]{article}
\usepackage{geometry}
\usepackage{amsthm}
\usepackage{amsmath}
\usepackage[ruled,algo2e,noend]{algorithm2e}
\usepackage{subcaption}

\fi

\ifcolt
\documentclass[anon,12pt,cleveref]{colt2024}
\usepackage{times}

\fi

\usepackage{svg}

\usepackage{enumitem}

\usepackage{graphicx}

\usepackage{amssymb}
\usepackage{hyperref} 
\usepackage{xcolor}
\usepackage{thmtools}
\usepackage{bbm}
\usepackage{tikz-cd}

\ifarxiv 
\usepackage[nameinlink,capitalize]{cleveref} 
\fi

\definecolor{dgreen}{rgb}{0,0.5,0}
\hypersetup{
  colorlinks=true,
  linkcolor=blue,
  filecolor=blue,
  citecolor = dgreen,      
  urlcolor=cyan,
}

\newenvironment{namedproof}[1]{\paragraph{Proof of #1.}\hspace{-1em}}{\hfill$\blacksquare$\vspace{1em}}

\ifarxiv
\newtheorem{theorem}{Theorem}[section]
\newtheorem{lemma}[theorem]{Lemma}

\newtheorem{corollary}[theorem]{Corollary} 
\newtheorem{conjecture}[theorem]{Conjecture}
\newtheorem{example}[theorem]{Example}
\fi

\ifarxiv
\theoremstyle{definition}
\newtheorem{definition}[theorem]{Definition}
\newtheorem{remark}[theorem]{Remark}
\fi
\newtheorem{setting}{Setting}

\DeclareMathOperator{\Var}{Var}
\newcommand{\norm}[1]{\left \lVert #1 \right \rVert}

\mathchardef\mhyphen="2D

\DeclareMathOperator*{\argmin}{argmin}

\newcommand{\RR}{\mathbb{R}}

\newcommand{\NN}{\mathbb{N}}

\DeclareMathOperator{\supp}{supp}

\DeclareMathOperator{\diag}{diag}
\DeclareMathOperator*{\E}{\mathbb{E}}
\DeclareMathOperator{\EE}{\mathbb{E}}

\DeclareMathOperator{\vspan}{span}

\DeclareMathOperator{\Unif}{Unif}

\DeclareMathOperator{\tr}{tr}

\DeclareMathOperator{\vrank}{rank}

\DeclareMathOperator*{\argmax}{arg\,max}

\renewcommand{\t}{\top}
\DeclareMathOperator{\poly}{poly}

\newcommand{\CS}{\mathcal{S}}
\newcommand{\CW}{\mathcal{W}}
\newcommand{\BX}{\mathbb{X}}
\newcommand{\MC}{\mathcal{C}}
\newcommand{\mslr}{{m_{\mathsf{SLR}}}}
\newcommand{\MA}{\mathcal{A}}

\newcommand{\BQ}{\mathbb{Q}}
\newcommand{\BP}{\mathbb{P}}
\DeclareMathOperator{\Rad}{Rad}

\newcommand{\tfinal}{{t_\mathsf{final}}}

\ifarxiv
\newcommand{\citep}[1]{\cite{{#1}}}
\fi

\newcommand{\imin}{{i_{\mathsf{min}}}}
\newcommand{\st}{\star}
\newcommand{\wh}{{\hat{w}}}
\newcommand{\wst}{{w^\st}}
\newcommand{\SLR}{{\mathsf{SLR}}}
\newcommand{\ME}{{\mathcal{E}}}
\newcommand{\dlow}{{d_{\mathsf{low}}}}
\newcommand{\dhigh}{{d_{\mathsf{high}}}}
\newcommand{\NP}{\mathsf{NP}}

\SetKwFunction{lp}{SmartScaling}
\SetKwFunction{rl}{RescaledLasso}

\allowdisplaybreaks

\ifarxiv
\title{Lasso with Latents: Efficient Estimation, Covariate Rescaling,\\ and Computational-Statistical Gaps}
\fi 
\ifcolt
\title[Lasso with Latents]{Lasso with Latents: Efficient Estimation, Covariate Rescaling,\\ and Computational-Statistical Gaps}
\fi 

\ifarxiv
\author{Jonathan A. Kelner\thanks{\texttt{kelner@mit.edu}. This work was supported in part by NSF Large CCF-1565235, NSF Medium CCF-1955217, and NSF TRIPODS 1740751.} \\ MIT \and Frederic Koehler\thanks{\texttt{fkoehler@uchicago.edu}. This work was supported in part by NSF award CCF-1704417, NSF award IIS-1908774, and N. Anari’s Sloan Research Fellowship} \\ University of Chicago \and Raghu Meka\thanks{\texttt{raghum@cs.ucla.edu}. This work was supported in part by NSF CAREER Award CCF-1553605 and NSF Small CCF-2007682} \\ UCLA \and Dhruv Rohatgi\thanks{\texttt{drohatgi@mit.edu}. This work was supported by a U.S. DoD NDSEG Fellowship.} \\ MIT}
\fi

\ifcolt
\coltauthor{
 \Name{Jonathan A. Kelner} \Email{kelner@mit.edu}\\
 \addr MIT
 \AND
 \Name{Frederic Koehler} \Email{fkoehler@uchicago.edu}\\
 \addr University of Chicago
 \AND
 \Name{Raghu Meka}
 \Email{raghum@cs.ucla.edu}\\
 \addr UCLA
 \AND
 \Name{Dhruv Rohatgi} \Email{drohatgi@mit.edu}\\
 \addr MIT
}
\fi

\begin{document}

\maketitle
\begin{abstract}
It is well-known that the statistical performance of Lasso can suffer significantly when the covariates of interest have strong correlations. In particular, the prediction error of Lasso becomes much worse than computationally inefficient alternatives like Best Subset Selection. Due to a large conjectured computational-statistical tradeoff in the problem of sparse linear regression, it may be impossible to close this gap in general.

In this work, we propose a natural sparse linear regression setting where strong correlations between covariates arise from unobserved latent variables. In this setting, we analyze the problem caused by strong correlations and design a surprisingly simple fix. While Lasso with standard normalization of covariates fails, there exists a heterogeneous scaling of the covariates with which Lasso will suddenly obtain strong provable guarantees for estimation. Moreover, we design a simple, efficient procedure for computing such a ``smart scaling.''

The sample complexity of the resulting ``rescaled Lasso'' algorithm incurs (in the worst case) quadratic dependence on the sparsity of the underlying signal. While this dependence is not information-theoretically necessary, we give evidence that it is optimal among the class of polynomial-time algorithms, via the method of low-degree polynomials. This argument reveals a new connection between sparse linear regression and a special version of sparse PCA with a \emph{near-critical negative spike}. The latter problem can be thought of as a real-valued analogue of learning a sparse parity. Using it, we also establish the first computational-statistical gap for the closely related problem of learning a Gaussian Graphical Model.
\end{abstract}

\section{Introduction}

Sparse linear regression (SLR) is one of the most fundamental problems in high-dimensional statistics. In this paper, we study algorithmic aspects of the problem.
For simplicity, we focus on the following setting with Gaussian random design (though our results should be generalizable to e.g., \ sub-Gaussian data, misspecification via oracle inequalities, etc.): 

\begin{definition}\label{def:slr-intro}
Let $\Sigma \in \RR^{n\times n}$ be positive semi-definite,  $w^\st \in \RR^n$ be $k$-sparse, and  $\sigma \ge 0$. We define $\SLR_{\Sigma,\sigma}(w^\st)$ to be the distribution of $(X,y)$ where $X \sim N(0,\Sigma)$ and $y \sim N(\langle X, w^\st \rangle,\sigma^2)$.
\end{definition}
Given $m$ independent samples $(X^{(j)},y^{(j)})_{j=1}^m$ from $\SLR_{\Sigma,\sigma}(w^\st)$, the goal of sparse linear regression is to produce an estimate $\hat w$ with low \emph{out-of-sample (clean) prediction error}, defined as:
\[\EE (\langle X^{(0)}, \hat w\rangle - \langle X^{(0)}, w^\st\rangle)^2 = (\hat w - w^\st)^\t \Sigma (\hat w - w^\st) =: \norm{\hat w - w^\st}_\Sigma^2\]
where $(X^{(0)},y^{(0)})$ is a fresh sample from $\SLR_{\Sigma,\sigma}(w^\st)$.

Despite significant effort, there is a vast gap in our understanding of the \emph{computational complexity} of sparse linear regression \--- and, in particular, how computational efficiency interplays with sample efficiency.
On the one hand, the natural \emph{Best Subset Selection} estimator \citep{hocking1967selection} achieves prediction error $O(\sigma^2 k (\log n)/m)$ with $m$ samples, so long as $m = \Omega(k\log n)$. Note that the sample complexity scales only logarithmically with the ambient dimension $n$, and no further assumptions on $\Sigma$ or $w^\st$ are needed. Unfortunately, this estimator is computationally intractable. On the other hand, classical estimators such as Lasso \citep{tibshirani1996regression} can be computed in polynomial time. However, they are known to have poor statistical performance (e.g., sample complexity \emph{linear} in $n$) in many settings where the covariates have strong correlations. In particular, Lasso is only statistically efficient when $\Sigma$ satisfies a restricted condition-number assumption such as the \emph{compatibility condition} \citep{van2009conditions}. While there are several special cases where Lasso fails but other polynomial-time algorithms are known to succeed, these are (thus far) the exceptions to the rule. See \Cref{sec:related-upper} for further discussion about Lasso and other estimators.

Given the dearth of strong algorithmic guarantees for SLR, it's natural to speculate that some choices of $\Sigma$ make SLR computationally hard for any sample complexity $m = o(n)$. But proving such a lower bound via average-case reduction or in any standard restricted computational model (e.g., low-degree polynomials or statistical queries) seems out of reach at present (see \Cref{sec:related-lower} for discussion of prior attempts). We lack even a \emph{conjecture} about which families of $\Sigma$ might induce computational hardness: obviously, $\Sigma$ must be ill-conditioned, but little else is clear. 

In this paper, we make progress on this problem by identifying the fundamental computational limits for a subclass of SLR problems. Informally, the subclass captures common situations where strong correlations are due to the existence of a few latent confounders or a few directions of unusually small variance in the data; see below for more details. For this subclass, we give mathematical evidence that no efficient estimator can succeed with significantly less than $O(k^2 \log n)$ samples (even though $O(k \log n)$ samples suffice information-theoretically), and we design a new polynomial-time algorithm that matches the lower bound.
Our efficient algorithm is based on a simple but surprisingly powerful \emph{smart scaling} procedure that we use as a preprocessing step to ``fix'' the Lasso. Our lower bound is 
based on a new connection between SLR and what we call the \emph{near-critical} regime of negatively spiked sparse principal component analysis (PCA).


\subsection{Upper bounds}     
We start by describing two natural settings where $\Sigma$ may be arbitrarily ill-conditioned (and Lasso has poor sample complexity and performs poorly empirically), but the degeneracies among covariates are sufficiently few or structured so that one may still hope for an efficient SLR algorithm.


\begin{setting}[Latent variable models]\label{setting:lvm}
Correlations are often induced via latent confounders. Thus, as is common in econometrics, causal inference, and other fields (see e.g. \cite{hoyle1995structural,pearl2009causality}), we 
can posit that covariates follow a \emph{Structural Equation Model} (SEM) with $h \ll n$ latent variables. Formally, we suppose that each observed covariate vector $X^{(j)}$ can be written as $X^{(j)} := AH^{(j)} + Z^{(j)}$ where $A$ is an unknown, fixed $n \times h$ matrix, $H^{(j)}$ is an unobserved Gaussian random vector, and independently $Z^{(j)} \sim N(0,D)$ is Gaussian noise where $D \succ 0$ is diagonal.
Thus, there is a rank-$h$ matrix $L \succeq 0$ such that each $X^{(j)}$ is a multivariate Gaussian with covariance matrix
\begin{equation}
\Sigma := \EE XX^\t = L + D. \label{eq:lvm-cov}
\end{equation}
\end{setting}

\begin{figure}[t]
\centering

\ifarxiv
\subcaptionbox{\label{fig:lvm}}[.4\textwidth]{\includegraphics[width=\linewidth]{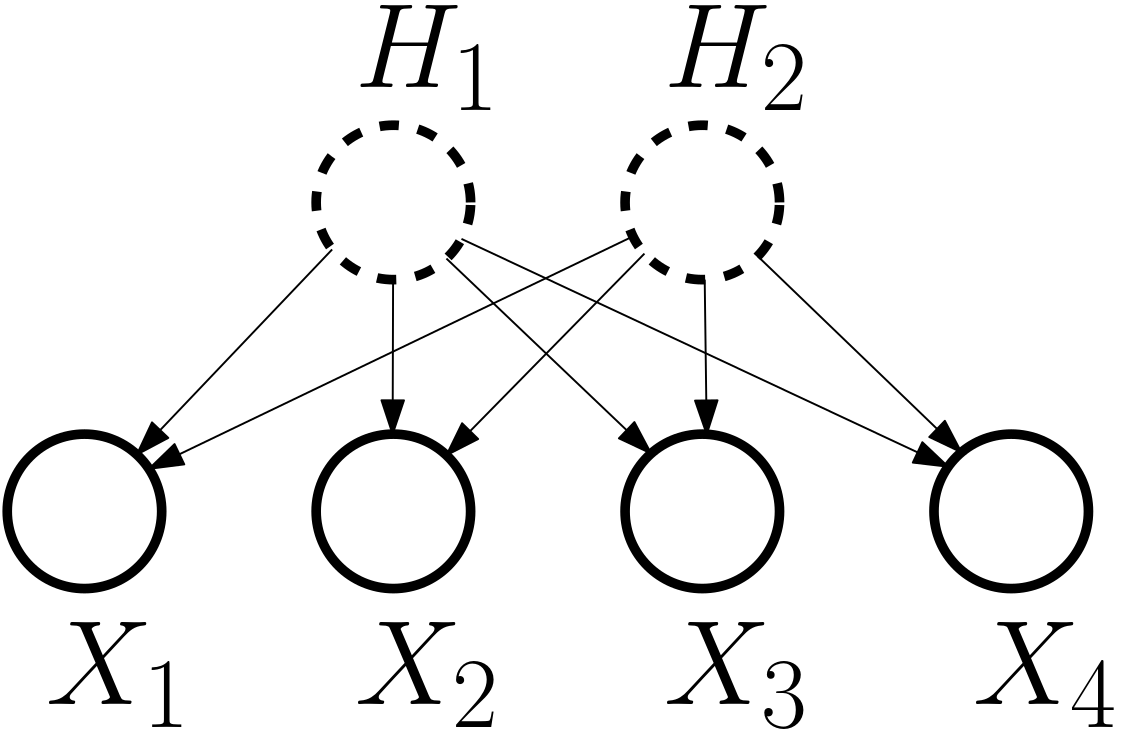}}
\hspace{5em}
\subcaptionbox{\label{fig:outlier-spec}}[.4\textwidth]{\includegraphics[width=\linewidth]{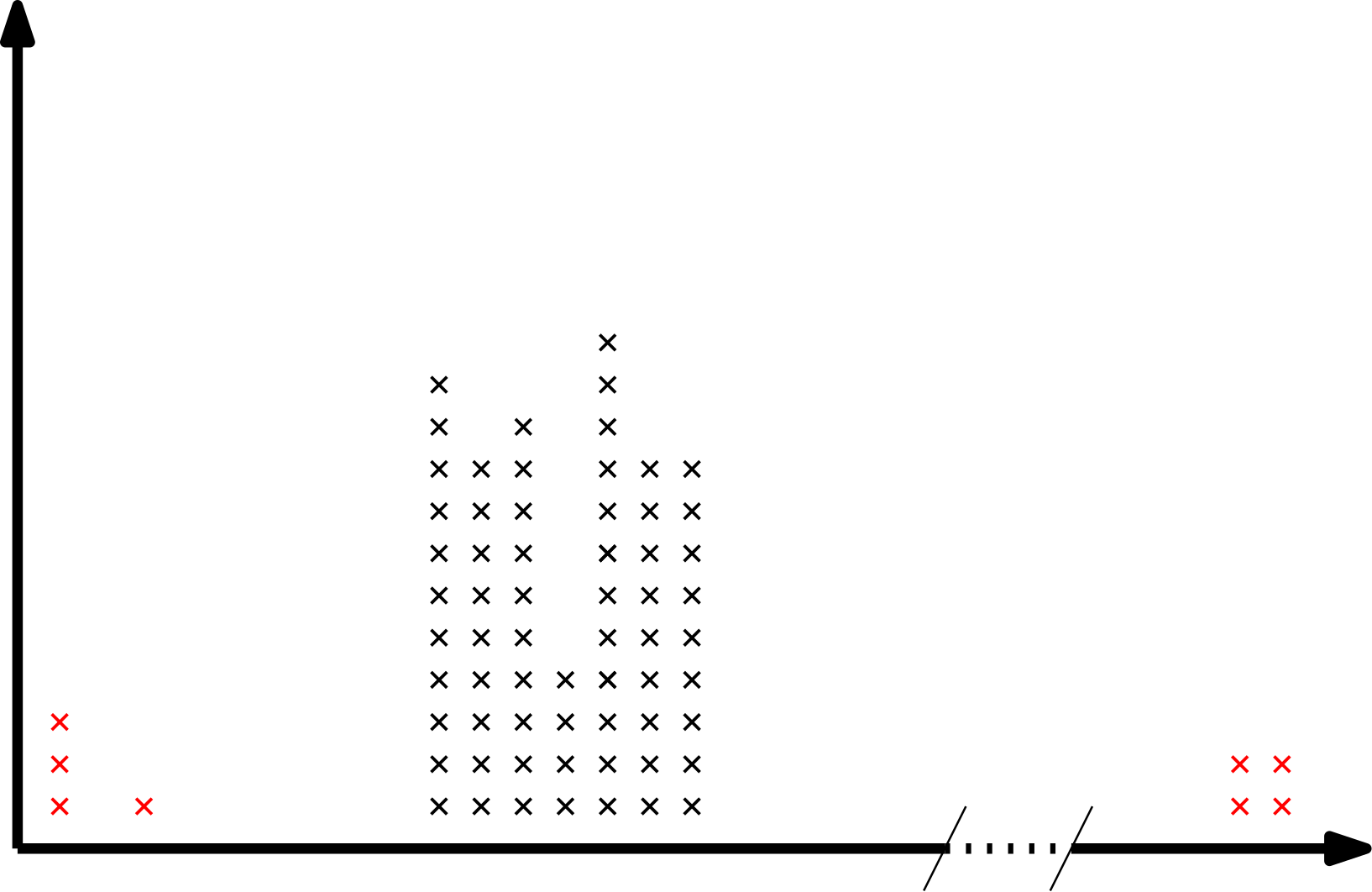}}
\fi

\ifcolt
\begin{subfigure}{}
\includesvg[width=12em]{lvm.svg}
\end{subfigure}
\hspace{5em}
\begin{subfigure}{}
\includesvg[width=12em]{spectrum.svg}
\end{subfigure}
\fi

\caption{(a) Example graphical model with $X_1,\dots,X_4$ observed and $H_1,H_2$ latent. (b) Example eigenspectrum that is well-conditioned aside from a few ``outliers'' (displayed in red).}
\label{fig}

\end{figure}

\begin{setting}[Eigenspectrum with outliers]\label{setting:outliers}
Alternatively, we may restrict the degeneracies in $\Sigma$ by explicitly controlling the eigenspectrum. In this setting, originally introduced by \cite{kelner2023feature}, we assume that the spectrum of $\Sigma$ is well-concentrated, aside from a small number of ``outlier'' eigenvalues. That is, suppose that the eigenvalues of $\Sigma$ are $\lambda_1 \leq \dots \leq \lambda_n$, and there is some $d \ll n$ such that $\lambda_{n-d}/\lambda_{d+1}$ is small (see \Cref{fig} for a depiction). Unlike in \cite{kelner2023feature}, we do not assume that $\Sigma$ is known.
\end{setting}
Note that the latter setting generalizes the classical, well-conditioned setting (where $\lambda_n/\lambda_1 = O(1)$). 
Both settings allow for a small number of approximate linear dependencies among the covariates, which is a natural case where Lasso may provably fail, requiring as many as $\Omega(n)$ samples to achieve non-trivial prediction error \cite[Theorem 6.5]{kelner2021power}.

\paragraph{Challenge: adapting to unknown structure.} In both settings, the covariates are drawn from a highly structured distribution, but one of the main challenges is that the structure is \emph{unknown}. In the latent variable model setting, $\Sigma$ has a ``low-rank plus diagonal'' decomposition. However, even if $\Sigma$ is known, efficiently computing such a decomposition is a well-studied open problem \citep{saunderson2012diagonal, bertsimas2017certifiably, wu2020high} with some evidence of intractability \citep{tunccel2023computational}. In the outlier setting, note that the eigendecomposition of $\Sigma$ isn't even identifiable from a sublinear number of samples. Thus, efficient sparse linear regression in these settings requires exploiting unknown structure without completely learning it.


\subsubsection{An efficient algorithm via smart scaling}


Conventional wisdom when applying the Lasso (and in statistics and machine learning more broadly)
is to scale all covariates to unit variance \citep{ahrens2020lassopack}.
While this is a good idea in many cases, it is not always the \emph{optimal} choice! 
In fact, in both settings described above, even if Lasso has poor performance with the standard scaling, there always exists a clever rescaling after which Lasso would achieve near-optimal sample complexity. We formalize this existence criterion via the following notion of \emph{$(\alpha,h)$-rescalability}. It essentially states that after rescaling by some diagonal matrix, covariates from $N(0,\Sigma)$ satisfy a restricted eigenvalue condition (similar to \cite{raskutti2010restricted}) modulo a low-rank subspace $\vspan(L)$.
\begin{definition}\label{def:quant-sparse}
For any $n \in \NN$ and $\gamma > 1$, let $\MC_n(\gamma) := \{x \in \RR^n: \norm{x}_1 \leq \gamma \norm{x}_\infty\}$ 
be the set of \emph{$\gamma$-quantitatively sparse vectors}.
\end{definition}
\begin{definition}\label{def:rescalable}
Let $n \in \NN$, and let $\Sigma \in \RR^{n\times n}$ be a positive semi-definite matrix. For $k,h\in \NN$ and $\alpha>0$, we say that $\Sigma$ is \emph{$(\alpha,h)$-rescalable at sparsity $k$} if there are matrices $D, L \in \RR^{n\times n}$ such that $D \succ 0$ is diagonal, $L \succeq 0$ has rank at most $h$, and
\begin{equation} I_n \preceq_{\MC_n(32k)} D^{-1/2}\Sigma D^{-1/2} \preceq \alpha I_n + L\label{eq:diag-prec-def} \end{equation}
where $I_n \preceq_{\MC_n(32k)} D^{-1/2}\Sigma D^{-1/2}$ means that $v^\t v \leq v^\t D^{-1/2}\Sigma D^{-1/2} v$ for all $v \in \MC_n(32k)$. If \Cref{eq:diag-prec-def} holds for all $k$, then we simply say that $\Sigma$ is $(\alpha,h)$-rescalable.
\end{definition}
In \Cref{setting:lvm} (the latent variable model with $h$ latent variables), it's immediate from \Cref{def:rescalable} that $\Sigma$ is $(1, h)$-rescalable. In \Cref{setting:outliers}, the implication is less obvious, but we are able to show that $\Sigma$ is $(\alpha,h)$-rescalable at sparsity $k$ with $\alpha = O(k^2 \lambda_{n-d}/\lambda_{d+1})$ and $h = O(k^2 d)$ (see \Cref{lemma:outlier-rescalability}). Thus, the notion of rescalability unifies both settings.

If $D$ were known, then one could simply rescale each sample via $X \mapsto D^{-1/2} X$. By standard analyses, Lasso with this ``oracle rescaling'' would have sample complexity $O((\alpha k + h)\log n)$, which is information-theoretically optimal for $\alpha, h = O(1)$. However, as discussed above, it is unreasonable to assume access to $D$, which in \Cref{setting:lvm} consists of the conditional variances of the covariates with respect to the unknown latent variables. Our first main result is a computationally efficient SLR algorithm \rl{} that doesn't need to know $D$ (or even $\Sigma$), and nonetheless matches the sample complexity of the ``oracle rescaled'' Lasso up to a factor of $k$:
\begin{theorem}\label{theorem:rescaled-lasso-intro}
Let $n,m,k,h \in \NN$ and $\alpha,\delta,\sigma,\lambda>0$. Suppose that $\Sigma \in \RR^{n\times n}$ is $(\alpha,h)$-rescalable at sparsity $k$, and $w^\st \in \RR^n$ is $k$-sparse. Let $(X^{(j)}, y^{(j)})_{j=1}^m \sim \SLR_{\Sigma,\sigma}(w^\st)$ be independent samples, and define $\hat w := \rl{\text{$(X^{(j)},y^{(j)})_{j=1}^m, k, \lambda$}}$ (see \Cref{alg:intro}). 

If $m = \Omega((\alpha k^2 + h)\log (n/\delta))$ and $\lambda = \Omega(\sigma \sqrt{(\alpha k^2 + h)\log(n/\delta) / (k^2 m)})$, then with probability at least $1-\delta$ it holds that $\norm{\hat w - w^\st}_\Sigma^2 \leq O(k^2\lambda^2)$. Moreover, the algorithm's time complexity is $\poly(n, \log \max_i \frac{\Sigma_{ii}}{D_{ii}})$, where $D$ is the (unknown) matrix in \Cref{def:rescalable}.
\end{theorem}
In particular, for the optimal choice of $\lambda$, the algorithm achieves prediction error $O(\sigma^2 (\alpha k^2 + h) \log(n)/m)$ with high probability. 
Note that the time complexity depends on $\max_i \Sigma_{ii}/D_{ii}$, but only logarithmically; hence, the algorithm runs in $\poly(n)$ time even if this ratio is exponentially large.
Applying \Cref{theorem:rescaled-lasso-intro} to \Cref{setting:lvm} is immediate using \Cref{eq:lvm-cov}; simply set $\alpha = 1$, and take $h$ to be the number of latent variables:
\begin{corollary}\label{theorem:lvm-lasso-intro}
Let $n,m,k,h \in \NN$ and $\delta,\sigma,\lambda>0$. Let $\Sigma := D+L \in \RR^{n\times n}$ for some diagonal matrix $D \succ 0$ and rank-$h$ matrix $L \succeq 0$, and let $w^\st \in \RR^n$ be $k$-sparse. Let $(X^{(j)}, y^{(j)})_{j=1}^m \sim \SLR_{\Sigma,\sigma}(w^\st)$ be independent samples, and define $\hat w := \rl{\text{$(X^{(j)},y^{(j)})_{j=1}^m, k, \lambda$}}$. If $m = \Omega((k^2 + h)\log (n/\delta))$ and $\lambda = \Omega(\sigma \sqrt{(k^2 + h)\log(n/\delta) / (k^2 m)})$, then with probability at least $1-\delta$ it holds that $\norm{\hat w - w^\st}_\Sigma^2 \leq O(k^2\lambda^2)$. Moreover, the time complexity is $\poly(n, \log \max_i \frac{\Sigma_{ii}}{D_{ii}})$.
\end{corollary}
Note the quadratic dependence on $k$ above and in \Cref{theorem:rescaled-lasso-intro}. While not information-theoretically necessary, we give evidence in \Cref{sec:lower-bounds} that it is the optimal dependence for efficient algorithms. 

The more involved application is to \Cref{setting:outliers}, where proving rescalability is non-trivial (see \Cref{lemma:outlier-rescalability}). Combining \Cref{theorem:rescaled-lasso-intro} with \Cref{lemma:outlier-rescalability} yields the following result. 
\begin{corollary}\label{theorem:outlier-lasso-intro}
Let $n,m,k \in \NN$ and $\delta,\sigma,\lambda>0$. Suppose that $\Sigma \in \RR^{n\times n}$ is positive definite with eigenvalues $\lambda_1 \leq \dots \leq \lambda_n$, and that $w^\st \in \RR^n$ is $k$-sparse. Let $(X^{(j)}, y^{(j)})_{j=1}^m$ be i.i.d.\ samples from $\SLR_{\Sigma,\sigma}(w^\st)$, and define $\hat w := \rl{\text{$(X^{(j)},y^{(j)})_{j=1}^m, k,\lambda$}}$. If $m = \Omega(\min_{0 \leq d < n} (k^4 \frac{\lambda_{n-d}}{\lambda_{d+1}} + k^2 d)\log (n/\delta))$ and $k\lambda = \Omega(\sigma \min_{0 \leq d < n} (k^2 \sqrt{\frac{\lambda_{n-d}}{\lambda_{d+1}}} + k\sqrt{d}) \sqrt{\log (n/\delta)/m})$, then with probability at least $1-\delta$ it holds that $\norm{\hat w - w^\st}_\Sigma^2 \leq O(k^2\lambda^2)$. Moreover, the algorithm's time complexity is $\poly(n, \log \frac{\lambda_n}{\lambda_1})$.
\end{corollary}
Hence for the optimal choice of $\lambda$, the algorithm achieves prediction error 
$O(\sigma^2 \min_{0 \leq d < n} (k^4 \frac{\lambda_{n-d}}{\lambda_{d+1}} + k^2 d)\log(n/\delta)/m)$. If the spectrum of $\Sigma$ has few outliers, in the sense that there is some $d = O(1)$ with $\lambda_{n-d}/\lambda_{d+1} = O(1)$, then this simplifies to $O(\sigma^2 k^4 \log(n)/m)$. This significantly improves upon the main result of \cite{kelner2023feature}; see \Cref{sec:related-upper} for more detailed comparison.

\subsection{Lower bounds}\label{sec:lower-bounds}

In light of \Cref{theorem:rescaled-lasso-intro}, $(\alpha,h)$-rescalable matrices $\Sigma$ (for small $\alpha, h$) are likely not the ``hardest'' covariance matrices, for which one might expect that no computationally efficient algorithm achieves non-trivial prediction error with $o(n)$ samples. However, there is still a polynomial gap between the sample complexity of \rl{} and the information-theoretic optimum: even for constant $\alpha$ and $h$, the sample complexity of \rl{} is $O(k^2 \log n)$ (to achieve prediction error $O(\sigma^2)$), whereas the inefficient Best Subset Selection estimator only requires $O(k\log n)$ samples. It is natural to ask whether this gap is inherent.

We prove that, under a plausible conjecture about the power of low-degree polynomials, the quadratic dependence on $k$ incurred by \rl{} may indeed be necessary for \emph{any} computationally efficient algorithm. While lower bounds have previously been shown for specific algorithms (such as Lasso and some generalizations), the below result is, to our knowledge, the first broad evidence for a (super-constant) computational-statistical tradeoff in sparse linear regression; see \Cref{sec:related-lower} for further discussion. 
\begin{theorem}\label{thm:no-subquadratic-alg-intro}
Let $\epsilon,C>0$ with $\epsilon \leq 2$. Let $\MA$ be a polynomial-time algorithm. Suppose that for any $n,k \in \NN$, $\sigma>0$, positive semi-definite, $(1,k)$-rescalable matrix $\Sigma \in \RR^{n\times n}$, $k$-sparse vector $w^\st \in \RR^n$, and $m \geq C k^{2-\epsilon}\log n$, the output $\hat w \gets \MA((X^{(j)},y^{(j)})_{j=1}^m)$ satisfies
\[\Pr[\norm{\hat w - w^\st}_\Sigma^2 \leq \sigma^2/10] \geq 1-o(1)\]
where the probability is over the randomness of $\MA$ and $m$ independent samples $(X^{(j)},y^{(j)})_{j=1}^m$ from $\SLR_{\Sigma,\sigma}(w^\st)$. Then \Cref{conj:ldlr} is false.
\end{theorem}
\Cref{conj:ldlr} is an instantiation of the \emph{Low-Degree Hypothesis}: it asserts that low-degree polynomials have optimal power among polynomial-time algorithms for a natural hypothesis testing problem called \emph{negative-spike sparse PCA}. Informally, this is the problem of distinguishing samples from the standard multivariate Gaussian $N(0,I_n)$ versus samples from the spiked Wishart model $N(0, I_n+\beta ww^\t)$, where $w$ is a random sparse unit vector, and $\beta \in (-1,0)$ is the spike strength.

The proof of \Cref{thm:no-subquadratic-alg-intro} has two components. First, we analyze the low-degree likelihood ratio for negative spike $k$-sparse PCA \--- thereby showing that low-degree polynomials require $\Omega(k^2)$ samples to solve the testing problem. We give additional evidence for the hardness of this problem by proving a lower bound for a natural SDP formulation. Second, we given an efficient reduction from this testing problem (in the \emph{near-critical} regime where $\beta$ is close to $-1$) to sparse linear regression with a $(1,k)$-rescalable covariance matrix $\Sigma$.\footnote{To contrast, \cite{bresler2018sparse} studied solving \emph{positive-spike} sparse PCA in the computationally easy regime by solving \emph{well-conditioned} SLR problems where the LASSO is statistically optimal up to constants. Our hardness reduction is crucially based upon the special properties of sparse PCA with a \emph{large (near-critical) negative spike}.} 

Our analysis of near-critical negative spike PCA also yields the first computational-statistical gap for learning Gaussian Graphical Models (GGMs). While it's information-theoretically possible to learn any $\kappa$-nondegenerate, degree-$d$ GGM with only $O(d\log(n)/\kappa^2)$ samples \citep{misra2017information}, the low-degree analysis implies (under \Cref{conj:ldlr}) that any computationally efficient algorithm requires at least $\Omega(d^{2-\epsilon}\log(n))$ samples, for any constant $\epsilon>0$ and even when $\kappa = \Omega(1)$. See \Cref{remark:ggms} for details. We do not know if this lower bound is tight for learning GGMs (the true computational-statistical gap may be much larger), but it is in fact tight for a natural \emph{testing} problem: testing between an empty graphical model and a sparse graphical model with at least one nonnegligible edge. The matching (computationally efficient) upper bound is given in \Cref{sec:testing}. 

\paragraph{Independent work.} In independent and concurrent work, Buhai, Ding, and Tiegel also gave evidence that sparse linear regression exhibits a $k$-to-$k^2$ computational-statistical gap \cite{buhai2024computationalstatistical}. Their proof proceeds along the same lines as ours (via reduction from negative-spike sparse PCA and analysis of the low-degree likelihood ratio).

\subsection{Outline}
In \Cref{sec:overview} we sketch the proofs of our main results. In \Cref{sec:related} we survey related work on algorithms and lower bounds for sparse linear regression and sparse PCA. In \Cref{sec:rescaling}, \Cref{sec:outliers}, and \Cref{sec:ldlr} we formally prove \Cref{theorem:rescaled-lasso-intro}, \Cref{theorem:outlier-lasso-intro}, and \Cref{thm:no-subquadratic-alg-intro} respectively. In \Cref{sec:testing}, we analyze testing between empty and non-empty GGMs, and in \Cref{sec:simulation} we show the results of applying \rl{} to a simple simulated dataset.

\section{Proof Overview}\label{sec:overview}

\input{overview}

\section{Related work}\label{sec:related}

\input{related}


\ifcolt 
\bibliography{bib}
\appendix
\newpage
\fi

\section{The rescaled Lasso}\label{sec:rescaling}

In this section we prove \Cref{theorem:rescaled-lasso-intro}. We start by analyzing the procedure \lp{}.

\begin{lemma}[Restatement of \Cref{lemma:lp-analysis-overview}]\label{lemma:lp-analysis}
Let $n,m,k \in \NN$. Let $\BX \in \RR^{m \times n}$. Suppose that
\begin{equation}
I_n \preceq_{\MC_n(32k)} D^{-1/2} \hat\Sigma D^{-1/2}
\label{eq:re-lb-assumption}
\end{equation}
where $D \succ 0$ is a diagonal matrix, and $\hat\Sigma := \frac{1}{m}\BX^\t \BX$. Then the algorithm $\lp{\text{$\BX, k$}}$
terminates after at most $T := n\log \max_{i \in [n]} \frac{2\hat\Sigma_{ii}}{D_{ii}}$ repetitions, and moreover the output $\hat{D} \in \RR^{n\times n}$ is a diagonal matrix satisfying the following properties:
\begin{itemize}
    \item $\hat{D} \succeq \frac{1}{2} D$.
    \item $\frac{1}{m}\norm{\BX v}_2^2 > 1$ for all $v \in \RR^n$ with $\norm{\hat D^{1/2} v}_\infty = 1$ and $\norm{\hat D^{1/2} v}_1 \leq 16k$.
\end{itemize}
\end{lemma}

\begin{proof}
For notational convenience, let $\tfinal$ be the step at which the algorithm returns the preconditioner.

\paragraph{Spectral lower bound.} First, we show by induction that for each $1 \leq t \leq \tfinal$, the intermediate result $\hat{D}^{(t)}$ satisfies $\hat{D}^{(t)} \succeq \frac{1}{2} D$. From \Cref{def:quant-sparse}, note that $\MC_n(32k)$ contains the standard basis vectors. Thus, we get from \Cref{eq:re-lb-assumption} that $1 \leq \hat\Sigma_{ii}/D_{ii}$ for all $i \in [n]$. Hence, $\hat{D}^{(1)} = \diag(\hat \Sigma) \succeq D$. This proves the base case of the induction.

Now fix any $1 \leq t < \tfinal$ and suppose that $\hat{D}^{(t)} \succeq \frac{1}{2} D$. We want to prove that $\hat{D}^{(t)}_{\imin\imin} \geq D_{\imin\imin}$. Suppose for the sake of contradiction that in fact 
\begin{equation} 
\hat{D}^{(t)}_{\imin\imin} < D_{\imin\imin}
\label{eq:d-reverse-assm}
\end{equation}
Then
\begin{align*} 
\norm{D^{1/2}v^{(t,\imin)}}_\infty
&\geq D^{1/2}_{\imin\imin} |v^{(t,\imin)}_\imin| \\ 
&\geq (\hat D^{(t)})^{1/2}_{\imin\imin}|v^{(t,\imin)}_\imin| \\ 
&\geq \frac{1}{16k} \norm{(\hat D^{(t)})^{1/2} v^{(t,\imin)}}_1 \\ 
&\geq \frac{1}{32k} \norm{D^{1/2} v^{(t,\imin)}}_1
\end{align*}
where the second inequality uses the assumption \Cref{eq:d-reverse-assm}, the third inequality uses the constraints in the program that defines $v^{(t,\imin)}$ (\Cref{eq:v-prog}), and the fourth inequality uses the induction hypothesis. It follows that $D^{1/2} v^{(t,\imin)} \in \MC_n(32k)$. So by \Cref{eq:re-lb-assumption}, we get that $\norm{v^{(t,\imin)}}_D^2 \leq \frac{1}{m}\norm{\BX v^{(t,\imin)}}_2^2$. Moreover since $t \neq \tfinal$, we have that $\frac{1}{m} \norm{\BX v^{(t,\imin)}}_2^2 \leq 1$. Thus,
\[\hat{D}^{(t)}_{\imin\imin} \cdot (v^{(t,\imin)})_\imin^2 = 1 \geq \frac{1}{m}\norm{\BX v^{(t,\imin)}}_2^2 \geq \norm{v^{(t,\imin)}}_D^2 \geq D_{\imin\imin} \cdot (v^{(t,\imin)})^2_{\imin}.\]
By \Cref{eq:v-prog} we know that $v^{(t,\imin)}_{\imin} \neq 0$. Simplifying the above display therefore gives that $\hat{D}^{(t)}_{\imin\imin} \geq D_{\imin\imin}$. This contradicts the assumption that $\hat{D}^{(t)}_{\imin\imin} < D_{\imin\imin}$, so in fact $\hat{D}^{(t)}_{\imin\imin} \geq D_{\imin\imin}$ holds unconditionally. By definition of $\hat{D}^{(t+1)}$ and the induction hypothesis, we get $\hat{D}^{(t+1)} \succeq \frac{1}{2} D$, which completes the induction.

\paragraph{Repetition bound.} Note that by definition we have $\det(\hat{D}^{(t)}) = 2^{1-t} \det(\hat{D}^{(1)})$ for all $1 \leq t \leq \tfinal$. Suppose that the algorithm requires more than $T$ repetitions, i.e. $\tfinal \geq T+1$. Then \[\det(\hat{D}^{(T+1)}) = 2^{-T} \det(\hat D^{(1)}) < \left(\min_{i \in [n]} \frac{D_{ii}}{2\hat\Sigma_{ii}}\right)^n \det(\hat D^{(1)})\]
by definition of $T$. But we have already seen that $\hat D^{(T+1)} \succeq \frac{1}{2}D$, so on the other hand
\[\frac{\det(\hat D^{(T+1)})}{\det(\hat D^{(1)})} \geq 2^{-n} \prod_{i=1}^n \frac{D_{ii}}{\hat\Sigma_{ii}} \geq \left(\min_{i \in [n]} \frac{D_{ii}}{2\hat\Sigma_{ii}}\right)^n.\]
This is a contradiction, so in fact the algorithm terminates after at most $T$ repetitions.

\paragraph{Restricted eigenvalue bound.} The output of the algorithm is $\hat D := \hat D^{(\tfinal)}$. Fix any $v \in \RR^n$ with $\norm{\hat D^{1/2} v}_\infty = 1$ and $\norm{\hat D^{1/2} v}_1 \leq 16k$. Since $\tfinal$ is the final step of the algorithm, by the termination condition it must be that $\frac{1}{m}\norm{\BX v^{(\tfinal,\imin)}}_2^2 > 1$ and thus, since $v$ is feasible for \Cref{eq:v-prog} for some $i \in [n]$, $\frac{1}{m}\norm{\BX v}_2^2 > 1$ as well.
\end{proof}

\paragraph{Notation.} For the remainder of \Cref{sec:rescaling}, we fix the following notation. Let $n,m,k,h \in \NN$ and $\alpha>0$. Let $\Sigma \in \RR^{n\times n}$ be a positive semi-definite matrix. We make the assumption that $\Sigma$ is $(\alpha,h)$-rescalable at sparsity $k$ (\Cref{def:rescalable}), i.e. 
\begin{equation} I_n \preceq_{\MC_n(32k)} D^{-1/2}\Sigma D^{-1/2} \preceq \alpha I_n + L\label{eq:sigma-rescalability}\end{equation}
where $D \succ 0$ is some diagonal matrix, and $L \succeq 0$ is some rank-$h$ matrix. 

Let $\sigma>0$, and let $w^\st \in \RR^n$ be $k$-sparse with support $S$. Let $(X^{(j)},y^{(j)})_{j=1}^m$ be $m$ independent samples from $\SLR_{\Sigma,\sigma}(w^\st)$. Let $\BX$ be the $m\times n$ matrix with rows $X^{(1)},\dots,X^{(m)}$, and let $\hat\Sigma := \frac{1}{m}\BX^\t \BX$.

\subsection{The good event: generalization,  concentration, and noise bounds}

The following definition states the event $\ME_{\ref{def:good-event}}(\delta)$ under which we will show that \rl{} (deterministically) has low prediction error. The event consists of three conditions, of which the first states that the covariates have accurate sample variances, the second is a uniform generalization bound, and the third bounds the bias of the noise term.

\begin{definition}\label{def:good-event}
Let $\delta>0$. Let $\ME_{\ref{def:good-event}}(\delta)$ be the event (over the samples $(X^{(j)},y^{(j)})_{j=1}^m$) that the following properties hold:
\begin{subequations}
\begin{itemize}
    \item For all $i \in [n]$,
    \begin{equation} \frac{1}{2}\Sigma_{ii} \leq \hat\Sigma_{ii} \leq 2\Sigma_{ii}.\label{eq:sigma-sigmahat-apx}\end{equation}
    \item For all $w \in \RR^n$,
    \begin{equation}\norm{w}_\Sigma^2 \leq 16\norm{w}_{\hat\Sigma}^2 + \frac{16\alpha  \log(48n/\delta)}{m} \norm{D^{1/2}w}_1^2.\label{eq:sigma-ub}\end{equation}
    \item For all $w \in \RR^n$,
    \begin{equation}\frac{1}{m}\langle w, \BX^\t (y - \BX w^\st)\rangle \leq \frac{\sigma}{\sqrt{m}}\left(2 \norm{D^{1/2}w}_1 \sqrt{\alpha\log(24n/\delta)} + \norm{w}_\Sigma\sqrt{2C_{\ref{lemma:unif-ip-bound}}h\log(24/\delta)}\right).\label{eq:noise-bound}\end{equation}
\end{itemize}
\end{subequations}
\end{definition}

Our first step is to show that $\ME_{\ref{def:good-event}}(\delta)$ holds with probability at least $1-\delta$. The first property \eqref{eq:sigma-sigmahat-apx} is standard and requires no additional assumptions on $\Sigma$, but the second and third properties both crucially use the spectral upper bound $D^{-1/2}\Sigma D^{-1/2} \preceq \alpha I_n + L$ guaranteed by rescalability. In particular, this upper bound implies (by Weyl's inequality and the fact that $L$ is low-rank) that $D^{-1/2}\Sigma D^{-1/2}$ has at most $h$ eigenvalues larger than $\alpha$. We can bound the corresponding eigenspaces separately to obtain the desired generalization bounds, a technique also applied in \cite{kelner2023feature} to deal with large outlier eigenvalues (and previously in other contexts \--- see the discussion in \cite{zhou2021optimistic} on ``covariance splitting'').

Concretely, the following theorem due to \cite{zhou2021optimistic} shows that to prove a uniform generalization bound (e.g. of the form \eqref{eq:sigma-ub}), it suffices to provide a uniform high-probability bound on $\sup_{w \in \RR^n} \langle w, x\rangle$ for $X \sim N(0,\Sigma)$. In \Cref{lemma:unif-ip-bound}, we use the technique of covariance splitting to derive such a bound for the rescalable setting. In \Cref{lemma:gen-bound}, we then invoke \Cref{theorem:zkss} to prove \eqref{eq:sigma-ub}.

\begin{theorem}[Theorem~1 in \cite{zhou2021optimistic}]\label{theorem:zkss}
Let $n,m \in \NN$ and $\epsilon,\delta > 0$. Let $\Sigma \in \RR^{n\times n}$ be a positive semi-definite matrix. Let $\BX \in \RR^{m \times n}$ have i.i.d. rows $X_1,\dots,X_m \sim N(0,\Sigma)$. Let $F: \RR^n \to [0,\infty]$ be a continuous function such that 
\[\Pr_{x \sim N(0,\Sigma)} [\sup_{w \in \RR^n} \langle w, x\rangle - F(w) > 0] \leq \delta.\]
If $m \geq 196\epsilon^{-2}\log(12/\delta)$, then with probability at least $1-4\delta$ it holds that for all $w \in \RR^n$, \[\norm{w}_\Sigma^2 \leq \frac{1+\epsilon}{m}\left(\norm{Xw}_2 + F(w)\right)^2.\]
\end{theorem}

\begin{lemma}\label{lemma:unif-ip-bound}
There is a constant $C_{\ref{lemma:unif-ip-bound}}$ with the following property. Let $n,h \in \NN$ and $\alpha>0$, and let $\Sigma \in \RR^{n \times n}$ be a positive semi-definite matrix. Suppose that $\Sigma \preceq \alpha D + L$ for some $D,L \in \RR^{n\times n}$ where $D \succ 0$ is diagonal and $L \succeq 0$ has rank at most $h$. Then
\[\Pr_{G \sim N(0, \Sigma)}\left[\forall w \in \RR^n: \langle w, G\rangle \leq \norm{D^{1/2} w}_1 \sqrt{2\alpha \log(16n/\delta)} + \norm{w}_\Sigma \sqrt{Ch\log(16/\delta)}\right] \geq 1-\delta.\]
\end{lemma}

\begin{proof}
Define the eigendecomposition of $D^{-1/2}\Sigma D^{-1/2}$ as $\sum_{i=1}^n \lambda_i u_i u_i^\t$. By assumption we have $D^{-1/2}\Sigma D^{-1/2} \preceq \alpha I_n + D^{-1/2}LD^{-1/2}$, and thus $D^{-1/2}(\Sigma - L) D^{-1/2} \preceq \alpha I_n$. Since $\vrank(D^{-1/2}LD^{-1/2}) \leq h$, Weyl's inequality (\Cref{lemma:weyl}) gives $\lambda_{n-h} \leq \alpha$.

Define the orthogonal projection matrix $P := \sum_{i=1}^{n-h} u_i u_i^\t$. Let $G \sim N(0,\Sigma)$. By the Gaussian maximal inequality, with probability at least $1-\delta/2$ over the draw $G$, we have
\begin{align*}
\norm{PD^{-1/2}G}_\infty 
&\leq \sqrt{\lambda_{\max}(PD^{-1/2}\Sigma D^{-1/2}P) \cdot 2\log(4n/\delta)} \\
&\leq \sqrt{2\lambda_{n-h}\log(4n/\delta)} \\
&\leq \sqrt{2\alpha\log(4n/\delta)}
\end{align*}
where the first inequality follows from the definition of $P$. Next, by \Cref{cor:hw}, if $C_{\ref{lemma:unif-ip-bound}}$ is chosen to be a sufficiently large constant, then with probability at least $1-\delta/2$,
\begin{align*}
\norm{\Sigma^{-1/2}D^{1/2} P^\perp D^{-1/2} G}_2
&\leq \sqrt{\tr(\Sigma^{-1/2}D^{1/2}P^\perp D^{-1/2} \Sigma D^{-1/2} P^\perp D^{1/2} \Sigma^{-1/2}) \cdot C_{\ref{lemma:unif-ip-bound}}\log(4/\delta)} \\
&= \sqrt{\tr(\Sigma^{1/2}D^{-1/2}P^\perp D^{1/2} \Sigma^{-1/2}) \cdot C_{\ref{lemma:unif-ip-bound}}\log(4/\delta)} \\
&= \sqrt{\tr(P^\perp) \cdot C_{\ref{lemma:unif-ip-bound}}\log(4/\delta)} \\
&= \sqrt{C_{\ref{lemma:unif-ip-bound}}h\log(4/\delta)}
\end{align*}
where the first equality uses the fact that $P^\perp = I_n - P$ commutes with $D^{-1/2}\Sigma D^{-1/2}$, and the second equality uses the cyclic property of trace. Consider the event (which occurs with probability at least $1-\delta$ over the draw $G$) that both of the above bounds hold. Then for any $w \in \RR^n$, we have
\begin{align*}
\langle w, G\rangle
&= \langle D^{1/2}w, D^{-1/2} G\rangle \\
&= \langle D^{1/2}w, P D^{-1/2} G\rangle + \langle D^{1/2}w, P^\perp D^{-1/2} G\rangle \\
&= \langle D^{1/2}w, PD^{-1/2} G\rangle + \langle \Sigma^{1/2}w,\Sigma^{-1/2}D^{1/2} P^\perp D^{-1/2} G\rangle \\
&\leq \norm{D^{1/2}w}_1 \sqrt{2\alpha\log(4n/\delta)} + \norm{w}_\Sigma \sqrt{C_{\ref{lemma:unif-ip-bound}}h\log(4/\delta)}.
\end{align*}
as needed.
\end{proof}

\begin{lemma}\label{lemma:gen-bound}
Let $\Sigma \in \RR^{n\times n}$ satisfy $\Sigma \preceq \alpha D+L$ for some diagonal matrix $D \succ 0$ and rank-$h$ matrix $L$. Let $\BX \in \RR^{m\times n}$ have i.i.d. rows $X_1,\dots,X_m \sim N(0,\Sigma)$. If $m \geq C(h+1)\log(96/\delta)$ for a sufficiently large constant $C$, then with probability at least $1-\delta$, it holds that for all $w \in \RR^n$,
\[ \norm{w}_\Sigma^2 \leq \frac{16}{m}\left(\norm{\BX w}_2^2 + \alpha\norm{D^{1/2}w}_1^2 \log(16n/\delta) \right).\] 
\end{lemma}

\begin{proof}
By \Cref{lemma:unif-ip-bound} and the fact that $m \geq 196\log(48/\delta)$, the hypothesis of \Cref{theorem:zkss} is satisfied $\epsilon := 1$, error probability $\delta/4$, and the functional $F$ defined as
\[F(w) := \norm{D^{1/2} w}_1 \sqrt{2\alpha \log(16n/\delta)} + \norm{w}_\Sigma\sqrt{C_{\ref{lemma:unif-ip-bound}}h\log(16/\delta)}.\]
The conclusion of \Cref{theorem:zkss} gives that with probability at least $1-\delta$, the following holds. For all $w \in \RR^n$, 
\begin{align*} \norm{w}_\Sigma^2 
&\leq \frac{2}{m}\left(\norm{\BX w}_2 + \norm{D^{1/2}w}_1\sqrt{2\alpha\log(16n/\delta)} + \norm{w}_\Sigma\sqrt{C_{\ref{lemma:unif-ip-bound}}h\log(16/\delta)}\right)^2 \\
&\leq \frac{8}{m}\left(\norm{\BX w}_2^2 + \alpha\norm{D^{1/2}w}_1^2 \log(16n/\delta) + \norm{w}_\Sigma^2 \cdot C_{\ref{lemma:unif-ip-bound}}h\log(16/\delta)\right) \\
&\leq \frac{16}{m}\left(\norm{\BX w}_2^2 + \alpha\norm{D^{1/2}w}_1^2 \log(16n/\delta)\right)
\end{align*}
where the last inequality holds by rearranging terms and using the fact that $m \geq 16C_{\ref{lemma:unif-ip-bound}}h\log(16/\delta)$ (so long as $C$ is chosen sufficiently large).
\end{proof}

We now prove that $\ME_{\ref{def:good-event}}(\delta)$ holds with probability at least $1-\delta$, observing that \Cref{lemma:unif-ip-bound} is exactly what is needed to prove the third property \eqref{eq:noise-bound}.

\begin{lemma}\label{lemma:good-event}
There is a constant $C_{\ref{lemma:good-event}}$ so that the following holds. Let $\delta>0$, and suppose that $m \geq C_{\ref{lemma:good-event}}(h+1)\log(288n/\delta)$. Then $\Pr[\ME_{\ref{def:good-event}}(\delta)] \geq 1-\delta$.
\end{lemma}

\begin{proof}
Since $m \geq 32\log(6n/\delta)$, we have by \Cref{lemma:chisq-conc} and a union bound that \Cref{eq:sigma-sigmahat-apx} holds for all $i \in [n]$, with probability at least $1-\delta/3$. By \Cref{lemma:gen-bound}, so long as $C_{\ref{lemma:good-event}}$ is a sufficiently large constant, we have that \Cref{eq:sigma-ub} holds for all $w \in \RR^n$, with probability at least $1-\delta/3$.

It remains to prove \Cref{eq:noise-bound}. Define the random variable $\xi := y - \BX w^\st$. Since $\norm{\xi}_2^2 \sim \sigma^2 \chi_m^2$, and $m \geq 8\log(12/\delta)$, it holds with probability at least $1-\delta/6$ that $\frac{1}{\sqrt{m}}\norm{\xi}_2 \leq \sigma \sqrt{2}$. Condition on $\xi$ and suppose that this event holds. Since $\xi$ is, by construction, independent of $\BX$, the random variable $\BX^\t \xi$ has distribution $N(0,\norm{\xi}_2^2 \Sigma)$. Thus, by \Cref{lemma:unif-ip-bound}, we have with probability at least $1-\delta/6$ over the randomness of $\BX$ that for all $w \in \RR^n$,
    \[ \left\langle w, \frac{\BX^t \xi}{\norm{\xi}_2} \right\rangle \leq \norm{D^{1/2}w}_1 \sqrt{2\alpha\log(24n/\delta)} + \norm{w}_\Sigma \sqrt{C_{\ref{lemma:unif-ip-bound}}h\log(24/\delta)}.\]
     Substituting in the bound on $\norm{\xi}_2$, we get that with probability at least $1-\delta/3$ (over $\BX$ and $\xi$), for all $w \in \RR^n$,
    \begin{equation}
    \frac{1}{m} \langle w, \BX^\t \xi\rangle \leq 2\sigma \norm{D^{1/2}w}_1 \sqrt{\frac{\alpha\log(24n/\delta)}{m}} + \sigma\norm{w}_\Sigma\sqrt{\frac{2C_{\ref{lemma:unif-ip-bound}}h\log(24/\delta)}{m}}
    \end{equation}
    which proves \Cref{eq:noise-bound}. By the union bound, $\ME_{\ref{def:good-event}}(\delta)$ holds with probability at least $1-\delta$.
\end{proof}

\subsection{Bounding the $\ell_1$ norm of the error}

Next, we assume that $\ME_{\ref{def:good-event}}(\delta)$ holds, and that the rescaling matrix $\hat D$ is (approximately) lower bounded by the oracle rescaling matrix $D$. Under these conditions, we derive a win-win (\Cref{lemma:l1-error-bound}) where either $\hat D^{1/2} (\hat w - w^\st)$ has small $\ell_1$ norm (where $\hat w$ is the rescaled Lasso solution) or $\hat w - w^\st$ is a violation to the second guarantee of \Cref{lemma:lp-analysis}.

Why is this a win-win? Ultimately, we need to bound the population variance $\norm{e}_\Sigma$ of the error $e := \hat w - w^\st$ in terms of the empirical variance $\norm{e}_{\hat \Sigma}$, and also bound the empirical variance (which would be trivially zero in the noiseless setting, but not in general). The first goal is partially accomplished by the generalization bound \Cref{eq:sigma-ub}, but it remains to bound $\lVert\hat D^{1/2} e\rVert_1$. The second goal also requires bounding $\lVert\hat D^{1/2} e\rVert_1$ in terms of $\norm{e}_{\hat \Sigma}$ (to bound the bias induced by the regularization term of the Lasso program). \Cref{lemma:l1-error-bound} will allow us to achieve both of these goals. 

We first need the following technical lemma, which should be thought of as a \emph{cone condition} for the error $e$ (compare to the noiseless case, where the analogous inequality is $\lVert\hat D^{1/2} e\rVert_1 \leq 2\lVert\hat D^{1/2} e_S\rVert_1 \leq 2k\lVert \hat D^{1/2} e\rVert_\infty$):


\begin{lemma}\label{lemma:cone-bound}
Let $\delta,\lambda>0$. Let $\hat D \in \RR^{n\times n}$ be a positive-definite diagonal matrix, and let $\hat w$ be a solution to the modified Lasso program
\begin{equation}
\hat w \in \argmin_{w \in \RR^n} \frac{1}{m} \norm{\BX w - y}_2^2 + \lambda \norm{\hat D^{1/2} w}_1.
\label{eq:what-def}
\end{equation}
If $\lambda \geq 64\sigma\sqrt{\frac{\alpha\log(24n/\delta)}{m}}$, event $\ME_{\ref{def:good-event}}(\delta)$ holds, and $D \preceq 64\hat D$, then $e := \hat w - w^\st$ satisfies
\[\lambda\norm{\hat D^{1/2} e}_1 + \frac{2}{m}\norm{\BX e}_2^2 \leq 4\lambda\norm{\hat D^{1/2} e_S}_1 + 4\sigma\norm{e}_\Sigma\sqrt{\frac{2C_{\ref{lemma:unif-ip-bound}}h\log(24/\delta)}{m}}.\]
\end{lemma}

\begin{proof}
First observe that
\begin{align}
\frac{1}{m}\norm{\BX\wh - y}_2^2 + \lambda\norm{\hat{D}^{1/2}(\wh-\wst)_{S^c}}_1 
&= \frac{1}{m}\norm{\BX\wh - y}_2^2 + \lambda\norm{\hat{D}^{1/2}\hat{w}_{S^c}}_1 \nonumber\\
&\leq \frac{1}{m}\norm{\BX\wst - y}_2^2 + \lambda\norm{\hat{D}^{1/2}\wst}_1 - \lambda\norm{\hat{D}^{1/2}\wh}_1 + \lambda\norm{\hat{D}^{1/2}\wh_{S^c}}_1 \nonumber\\
&= \frac{1}{m}\norm{\BX\wst - y}_2^2 + \lambda\norm{\hat{D}^{1/2}\wst}_1 - \lambda\norm{\hat{D}^{1/2}\wh_S}_1 \nonumber\\
&\leq \frac{1}{m}\norm{\BX\wst-y}_2^2 + \lambda\norm{\hat{D}^{1/2}(\wh-\wst)_S}_1
\label{eq:sc-l1-error-0}
\end{align}
where the first inequality is by optimality of $\wh$ in \Cref{eq:what-def} and the second inequality is by reverse triangle inequality. It follows that
\begin{align*}
\lambda\norm{\hat{D}^{1/2}e}_1
&= \lambda\norm{\hat{D}^{1/2}e_S}_1 + \lambda\norm{\hat{D}^{1/2}e_{S^c}}_1 \\
&\leq \frac{1}{m}\norm{\BX\wst - y}_2^2 - \frac{1}{m}\norm{\BX\wh - y}_2^2 + 2\lambda\norm{\hat{D}^{1/2}e_S}_1 \\
&= -\frac{1}{m}\norm{\BX e}_2^2 - \frac{2}{m} \langle \BX e, \BX\wst - y\rangle + 2\lambda\norm{\hat{D}^{1/2}e_S}_1
\end{align*}
where the first inequality uses \Cref{eq:sc-l1-error-0}, and the second equality expands $\norm{\BX\wh-y}_2^2 = \norm{\BX e + (\BX \wst - y)}_2^2$. Now since the event $\ME_{\ref{def:good-event}}(\delta)$ was assumed to hold, we can apply \Cref{eq:noise-bound} with vector $e$ to get that
\begin{align*}
-\frac{2}{m}\langle\BX e,\BX w^\st - y\rangle 
&\leq 4\sigma\norm{D^{1/2}e}_1\sqrt{\frac{\alpha\log(24n/\delta)}{m}} + 2\sigma\norm{e}_\Sigma\sqrt{\frac{2C_{\ref{lemma:unif-ip-bound}}h\log(24/\delta)}{m}} \\ 
&\leq \frac{\lambda}{2}\norm{\hat D^{1/2} e}_1 + 2\sigma \norm{e}_\Sigma \sqrt{\frac{2C_{\ref{lemma:unif-ip-bound}}h\log(24/\delta)}{m}}
\end{align*}
where the second inequality uses the lemma assumptions that $\lambda \geq 64\sigma\sqrt{\frac{\alpha\log(24n/\delta)}{m}}$ and $D \preceq 64\hat{D}$  (and $D$, $\hat D$ are diagonal). Substituting into the previous display and rearranging terms gives that
\[\frac{\lambda}{2}\norm{\hat D^{1/2}e}_1 + \frac{1}{m}\norm{\BX e}_2^2 \leq 2\lambda \norm{\hat D^{1/2} e_S}_1 + 2\sigma\norm{e}_\Sigma \sqrt{\frac{2C_{\ref{lemma:unif-ip-bound}}h\log(24/\delta)}{m}}\]
which completes the proof.
\end{proof}

\begin{lemma}\label{lemma:l1-error-bound}
Let $\delta,\lambda>0$. Let $\hat D \in \RR^{n\times n}$ be a positive-definite diagonal matrix, and let $\hat w$ be a solution to the modified Lasso program \Cref{eq:what-def}. Suppose that event $\ME_{\ref{lemma:good-event}}(\delta)$ holds, that $D \preceq 64\hat D$, that $m \geq 128C_{\ref{lemma:unif-ip-bound}}h\log(24/\delta)$, and that $\lambda \geq \max\left(64\sigma\sqrt{\frac{\alpha \log(48n/\delta)}{m}}, \frac{32\sigma\sqrt{2C_{\ref{lemma:unif-ip-bound}}h\log(24/\delta)}}{k\sqrt{m}}\right)$. Then $e := \hat w - w^\st$ satisfies at least one of the following properties:
\begin{enumerate}[label=(\alph*)]
\item $\norm{\hat{D}^{1/2}e}_1 \leq \frac{20k}{\sqrt{m}}\norm{\BX e}_2,$ or
\item $\frac{1}{m}\norm{\BX e}_2^2 \leq \norm{\hat D^{1/2} e}_\infty^2$ and $\hat{D}^{1/2} e \in \MC_n(9k)$.
\end{enumerate}
\end{lemma}

\begin{proof}
By \Cref{lemma:cone-bound} (dropping the second term on the left-hand side), we have \allowdisplaybreaks
\begin{align*}
\lambda\norm{\hat{D}^{1/2}e}_1
&\leq 4\lambda\norm{\hat{D}^{1/2}e_S}_1 + 4\sigma\norm{e}_\Sigma\sqrt{\frac{2C_{\ref{lemma:unif-ip-bound}}h\log(24/\delta)}{m}}\\
&\leq 4k\lambda\norm{\hat{D}^{1/2}e}_\infty + \frac{16\sigma\sqrt{2C_{\ref{lemma:unif-ip-bound}}h\log(24/\delta)}}{m}\left(\norm{\BX e}_2 + \norm{D^{1/2} e}_1\sqrt{\alpha\log(48n/\delta)}\right) \\
&\leq 4k\lambda\norm{\hat{D}^{1/2}e}_\infty + \frac{16\sigma\sqrt{2C_{\ref{lemma:unif-ip-bound}}h\log(24/\delta)}}{m}\norm{\BX e}_2 + \frac{\lambda}{32}\norm{D^{1/2} e}_1
\end{align*}
where the second inequality uses $k$-sparsity of $e_S$ to bound $\norm{\hat D^{1/2} e_S}_1$, and \Cref{eq:sigma-ub} (along with the inequality $\sqrt{a+b}\leq \sqrt{a}+\sqrt{b}$) to bound $\norm{e}_\Sigma$; and the third inequality uses the lemma assumptions that $\lambda \geq 64\sigma\sqrt{\frac{\alpha\log(48n/\delta)}{m}}$ and $m \geq 128C_{\ref{lemma:unif-ip-bound}}h\log(24/\delta)$. Since $D \preceq 64\hat D$ and $D$, $\hat D$ are diagonal, we have $\norm{D^{1/2} e}_1 \leq 8\norm{\hat D^{1/2} e}_1$. Substituting this bound above and rearranging terms gives
\begin{align} \lambda\norm{\hat{D}^{1/2}e}_1 
&\leq 8k\lambda\norm{\hat{D}^{1/2}e}_\infty + \frac{32\sigma\sqrt{2C_{\ref{lemma:unif-ip-bound}}h\log(24/\delta)}}{m}\norm{\BX e}_2 \nonumber \\
&\leq 8k\lambda\norm{\hat{D}^{1/2}e}_\infty + \frac{k\lambda}{\sqrt{m}}\norm{\BX e}_2
\label{eq:l1-linfty-bound-0}
\end{align}
where the second inequality is by the lemma assumption that $\lambda \geq \frac{32\sigma\sqrt{2C_{\ref{lemma:unif-ip-bound}}h\log(24/\delta)}}{k\sqrt{m}}$. Now suppose that property (a) of the lemma statement fails to hold, i.e.
\[ \frac{1}{\sqrt{m}} \norm{\BX e}_2 < \frac{1}{20k}\norm{\hat{D}^{1/2} e}_1.\]
Substituting into the right-hand side of \Cref{eq:l1-linfty-bound-0} and rearranging terms gives $\lambda\norm{\hat{D}^{1/2}e}_1 \leq 9k\lambda\norm{\hat{D}^{1/2}e}_\infty.$
As a result, we have $\hat D^{1/2} e \in \MC_n(9k)$ and also
\[\frac{1}{m}\norm{\BX e}_2^2 \leq \frac{1}{400k^2}\norm{\hat{D}^{1/2}e}_1^2 \leq \norm{\hat{D}^{1/2}e}_\infty^2.\]
Thus, property (b) of the lemma statement holds.
\end{proof}

\subsection{Putting everything together}

The following lemma shows that under the good event $\ME_{\ref{def:good-event}}(\delta)$ (which, as we've shown, occurs with high probability), the assumed spectral lower bound on $\Sigma$ in \Cref{eq:sigma-rescalability} transfers to a lower bound on the empirical covariance $\hat\Sigma$.

\begin{lemma}\label{lemma:transfer-lb}
Let $\delta>0$ and suppose that $m \geq 2^{15}\alpha k^2\log(48n/\delta)$. In the event $\ME_{\ref{def:good-event}}(\delta)$, we have that
\begin{equation} I_n \preceq_{\MC_n(32k)} 32D^{-1/2}\hat\Sigma D^{-1/2}.\label{eq:emp-cov-lb-guarantee}\end{equation}
\end{lemma}

\begin{proof}
Fix $v \in \MC_n(32k)$. Applying \Cref{eq:sigma-ub} to the vector $w := D^{-1/2}v$, we have
\begin{align*} \norm{D^{-1/2}v}_\Sigma^2 &\leq 16\norm{D^{-1/2} v}_{\hat\Sigma}^2 + \frac{16\alpha\log(48n/\delta)}{m}\norm{v}_1^2 \\ 
&\leq 16\norm{D^{-1/2} v}_{\hat\Sigma}^2 + \frac{2^{14}\alpha k^2 \log(48n/\delta)}{m} \norm{v}_2^2 \\ 
&\leq 16\norm{D^{-1/2} v}_{\hat\Sigma}^2 + \frac{1}{2}\norm{D^{-1/2}v}_\Sigma^2
\end{align*}
where the second inequality uses that $v \in \MC_n(32k)$ (along with $\norm{\cdot}_\infty \leq \norm{\cdot}_2$), and the third inequality uses the assumption on $m$ together with \Cref{eq:sigma-rescalability}. Simplifying and once again using \Cref{eq:sigma-rescalability}, we get
\[\norm{v}_2^2 \leq \norm{D^{-1/2}v}_\Sigma^2 \leq 32\norm{D^{-1/2}v}_{\hat\Sigma}^2\]
as needed.
\end{proof}

Note that \eqref{eq:emp-cov-lb-guarantee} is exactly the condition needed to apply \Cref{lemma:lp-analysis}. We now have all the pieces needed to bound the prediction error of \rl{} under the rescalability assumption. We restate the notation for completeness.

\begin{theorem}[Restatement of \Cref{theorem:rescaled-lasso-intro}]\label{theorem:rescaled-lasso}
Let $n,m,k,h \in \NN$ and $\alpha,\sigma,\delta,\lambda>0$. Let $\Sigma \in \RR^{n\times n}$ be a positive definite matrix that is $(\alpha,h)$-rescalable at sparsity $k$ (\Cref{def:rescalable}), and let $w^\st \in \RR^n$ be a $k$-sparse vector. Let $(X^{(j)},y^{(j)})_{j=1}^m$ be $m$ independent samples from $\SLR_{\Sigma,\sigma}(w^\st)$, and define the estimator $\hat w := \rl{\text{$(X^{(j)},y^{(j)})_{j=1}^m$, $k$, $\lambda$}}$. Suppose that \[\lambda \geq \max\left(64\sigma\sqrt{\frac{\alpha \log(48n/\delta)}{m}}, \frac{32\sigma\sqrt{2C_{\ref{lemma:unif-ip-bound}}h\log(24/\delta)}}{k\sqrt{m}}\right)\] and $m \geq C_{\ref{theorem:rescaled-lasso}}(h+\alpha k^2)\log(288n/\delta)$ for a sufficiently-large constant $C_{\ref{theorem:rescaled-lasso}}$. Then
\[\Pr[\norm{\wh - \wst}_\Sigma \leq 318k\lambda] \geq 1-\delta.\]
Moreover, the time complexity of \rl{} is $\poly(n, \log \max_i \frac{\Sigma_{ii}}{D_{ii}})$, where $D\succ 0$ is the diagonal matrix guaranteed by \Cref{def:rescalable}.
\end{theorem}

\begin{proof}
As before, let $\BX$ be the $m \times n$ random matrix with rows $X^{(1)},\dots,X^{(m)}$, and let $\hat\Sigma := \frac{1}{m}\sum_{j=1}^m X^{(j)} (X^{(j)})^\t$ be the empirical covariance matrix. By \Cref{def:rescalable}, there is a diagonal matrix $D \succ 0$ and a rank-$h$ matrix $L \succeq 0$ such that $I_n \preceq_{\MC_n(32k)} D^{-1/2} \Sigma D^{-1/2} \preceq \alpha I_n + L$. By \Cref{lemma:good-event} and choice of $m$, the event $\ME_{\ref{def:good-event}}(\delta)$ occurs with probability at least $1-\delta$. From now on, let us condition on this event. By \Cref{lemma:transfer-lb} and choice of $m$, we have that \[I_n \preceq_{\MC_n(32k)} 32D^{-1/2}\hat\Sigma D^{-1/2}.\] Thus, we can apply \Cref{lemma:lp-analysis} with the sample matrix $\BX$, sparsity $k$, and rescaling matrix $\frac{1}{32} D$. We get that the time complexity of \lp{$\BX$,$k$} is $\poly(n,\log \max_i \frac{\Sigma_{ii}}{D_{ii}})$ (using that $\hat\Sigma_{ii} \leq \Sigma_{ii}$ for all $i \in [n]$ under event $\ME_{\ref{def:good-event}}(\delta)$); this implies the claimed time complexity bound for \rl{}. Moreover, the output $\hat D$ satisfies $D \preceq 64 \hat D$ and also $\frac{1}{m}\norm{\BX v}_2^2 > \norm{\hat D^{1/2} v}_\infty^2$ for all $v \in \RR^n$ with $\hat D^{1/2} v \in \MC_n(16 k)$.

We now apply \Cref{lemma:l1-error-bound} with this choice of $\hat D$, using the bound $D \preceq 64\hat D$ as well as the assumptions on $m$ and $\lambda$. Note that property (b) of \Cref{lemma:l1-error-bound} contradicts our previously-derived guarantee on $\hat D$. Thus, property (a) must hold, i.e. $\norm{\hat D^{1/2} e}_1 \leq \frac{20k}{\sqrt{m}} \norm{\BX e}_2$, where $e := \hat w - w^\st$ is the error of the output of the algorithm. We now apply \Cref{eq:sigma-ub} to the vector $e$, which yields
\begin{align}
\norm{e}_\Sigma^2 
&\leq \frac{16}{m}\norm{\BX e}_2^2 + \frac{16\alpha\log(48n/\delta)}{m} \norm{D^{1/2}e}_1^2 \nonumber\\
&\leq \frac{16}{m}\norm{\BX e}_2^2 + \frac{2^{10}\alpha\log(48n/\delta)}{m} \norm{\hat{D}^{1/2}e}_1^2 \nonumber\\
&\leq \frac{16}{m}\norm{\BX e}_2^2 + \frac{2^{10}\alpha\log(48n/\delta)}{m} \cdot \frac{400k^2}{m} \norm{\BX e}_2^2 \nonumber\\
&\leq \frac{32}{m}\norm{\BX e}_2^2 \label{eq:what-generalization-0}
\end{align}
where the second inequality uses the bound $D \preceq 64\hat D$ (together with the fact that $D$, $\hat D$ are diagonal), and the final inequality is by choice of $m$. But by \Cref{lemma:cone-bound} (this time, dropping the first term of the left-hand side), we can conversely bound $\frac{2}{m}\norm{\BX e}_2^2$ as
\begin{align*}
\frac{2}{m}\norm{\BX e}_2^2
&\leq 4\lambda \norm{\hat D^{1/2} e_S}_1 + 4\sigma \norm{e}_\Sigma \sqrt{\frac{2C_{\ref{lemma:unif-ip-bound}}h\log(24/\delta)}{m}} \\ 
&\leq \frac{80k\lambda}{\sqrt{m}} \norm{\BX e}_2 + \frac{32\sigma \sqrt{C_{\ref{lemma:unif-ip-bound}}h\log(24/\delta)}}{m} \norm{\BX e}_2 \\ 
&\leq \frac{112k\lambda}{\sqrt{m}} \norm{\BX e}_2,
\end{align*}
where the second inequality again uses the bound $\norm{\hat D^{1/2} e}_1 \leq \frac{20k}{\sqrt{m}}\norm{\BX e}_2$, as well as \Cref{eq:what-generalization-0}, and the third inequality uses the assumption that $\lambda \geq \sigma\sqrt{\frac{C_{\ref{lemma:unif-ip-bound}}h\log(24/\delta)}{k^2 m}}$. Hence, dividing out by $\frac{2}{\sqrt{m}}\norm{\BX e}_2$ and combining with \Cref{eq:what-generalization-0}, we see that
\[\norm{e}_\Sigma \leq \frac{4\sqrt{2}}{\sqrt{m}}\norm{\BX e}_2 \leq 318k\lambda\]
as claimed.
\end{proof}



\section{Rescalability of covariance with few outlier eigenvalues}\label{sec:outliers}

In this section we prove \Cref{theorem:outlier-lasso-intro}, restated (in slightly greater generality) below. It asserts that \rl{} is sample-efficient for sparse linear regression when the covariance matrix has few outlier eigenvalues (and runs in polynomial time under the mild assumption that the condition number of $\Sigma$ is at most exponential in $n$).

\begin{theorem}[Restatement of \Cref{theorem:outlier-lasso-intro}]\label{theorem:outlier-lasso}
There is a constant $C_{\ref{theorem:outlier-lasso}}$ so that the following holds. Let $n,m,k,\dlow,\dhigh \in \NN$ with $\dlow+\dhigh < n$. Let $\sigma,\delta,\lambda>0$. Let $\Sigma \in \RR^{n\times n}$ be a positive definite matrix with eigenvalues $\lambda_1 \leq \dots \leq \lambda_n$, and let $w^\st \in \RR^n$ be a $k$-sparse vector. Let $(X^{(j)},y^{(j)})_{j=1}^m$ be $m$ independent samples from $\SLR_{\Sigma,\sigma}(w^\st)$, and define the estimator $\hat w := \rl{\text{$(X^{(j)},y^{(j)})_{j=1}^m$, $k$, $\lambda$}}$. Suppose that \[k\lambda \geq \frac{C_{\ref{theorem:outlier-lasso}}\sigma}{\sqrt{m}} \left(k^2 \sqrt{\frac{\lambda_{n-\dhigh}}{\lambda_{\dlow+1}}} + k\sqrt{\dlow} + \sqrt{\dhigh}\right) \sqrt{\log(48n/\delta)}\] and $m \geq C_{\ref{theorem:outlier-lasso}}(k^4 \frac{\lambda_{n-\dhigh}}{\lambda_{\dlow+1}} + k^2 \dlow + \dhigh)\log(288n/\delta)$. Then
\[\Pr[\norm{\wh - \wst}_\Sigma \leq 318k\lambda] \geq 1-\delta.\]
Moreover, the time complexity of \rl{} is $\poly(n, \log(\lambda_n/\lambda_1))$.
\end{theorem}

Note that $\dlow$ quantifies the number of outlier eigenvalues at the ``low'' end of the eigenspectrum, and $\dhigh$ quantifies the number of outliers at the ``high'' end; moreover, \rl{} is agnostic to these parameters, meaning that it automatically achieves the statistical accuracy guaranteed by the optimal choice of $\dlow$ and $\dhigh$ above (so long as $\lambda$ is chosen appropriately, which is also important for the standard Lasso). For simplicity, we stated \Cref{theorem:outlier-lasso-intro} only for the special case $\dlow = \dhigh = d$ (that nonetheless captures the essence of the result). 

This theorem will follow immediately from our generic analysis of \rl{} (\Cref{theorem:rescaled-lasso}) once we prove that any covariance matrix with few outlier eigenvalues is $(\alpha,h)$-rescalable (\Cref{def:rescalable}) with appropriate parameters $\alpha$ and $h$. Towards this end, the following lemma is key. Essentially, it states that if $\Sigma$ is a covariance matrix with at most $d$ ``small'' eigenvalues, then there is a rescaling $\tilde D$ that touches only $O(dk^2)$ coordinates, but makes every approximate linear dependency among the covariates $\Omega(k)$-quantitatively dense. The matrix $D$ needed for \Cref{def:rescalable} will then be an appropriate scalar multiple of $\tilde D$, and the existence of a low-rank matrix $L$ satisfying the upper bound in \Cref{def:rescalable} will use the fact that most entries of $\tilde D$ are equal to one.

\begin{lemma}\label{lemma:outlier-diag-construction}
Let $n,k \in \NN$, and let $d$ be an integer with $0 \leq d < n$. Let $\Sigma \in \RR^{n\times n}$ be a positive-definite matrix with eigenvalues $\lambda_1 \leq \dots \leq \lambda_n$. Then there is a diagonal matrix $\tilde{D} \in \RR^{n\times n}$ with $0 \prec \tilde{D} \preceq I_n$ satisfying the following properties:
\begin{itemize}
    \item $|\{i: \tilde{D}_{ii} \neq 1\}| \leq 128dk^2$ 
    \item For every $v \in \RR^n$ with $\norm{\tilde{D}^{1/2}v}_2 \leq (k/2)\norm{\tilde{D}^{1/2} v}_\infty$, it holds that $\norm{\tilde{D}^{1/2} v}_2 \leq 4k \lambda_{d+1}^{-1/2} \norm{v}_\Sigma.$
\end{itemize}
\end{lemma}
\begin{remark}[Comparison to \cite{kelner2023feature}]
It is useful to compare \Cref{lemma:outlier-diag-construction} with \citep[Lemma 2.4]{kelner2023feature}, the main structural lemma underlying the prior algorithm for sparse linear regression in \Cref{setting:outliers}. Their lemma may be interpreted as constructing a rescaling matrix with \emph{binary} entries, such that all but $k^{O(k)} d$ entries are equal to one, and there are no sparse approximate linear dependencies. \Cref{lemma:outlier-diag-construction} relaxes the binary assumption, and thereby achieves a much stronger guarantee both quantitatively \--- in the number of entries not equal to one \--- and qualitatively \--- in that it rules out quantitatively sparse dependencies, not just algebraically sparse dependencies. The first improvement is the source of our improved sample complexity, but the second improvement is even more crucial: rescalability requires a spectral lower bound that holds for all quantitatively sparse vectors. Thus, \citep[Lemma 2.4]{kelner2023feature} does not have any direct implication for rescalability.
\end{remark}

The proof of \Cref{lemma:outlier-diag-construction} is constructive, and looks quite similar to an ``oracle'' version of \lp{} where the procedure is given access to $\Sigma$, and thus can compute an eigendecomposition (but of course, this is only within the analysis). The intuition is as follows. Drawing on the proof technique of \cite{kelner2023feature}, one might hope that there is a good \emph{binary} rescaling matrix $\tilde D$ (and that it can be constructed by iteratively zeroing out ``bad'' coordinates). Unfortunately, the following example shows that this is false:

\begin{example}\label{example:decay-vec}
Let $v := (1, 1/2, 1/4,\dots, 1/2^{n-1}) \in \RR^n$, and let $\Sigma := I_n - (1-\epsilon) \frac{vv^\t}{\norm{v}_2^2}$, so that $\Sigma$ has a single small eigenvalue, with eigenvector $v$. Note that $v$ is $O(1)$-quantitatively sparse, and $\norm{v}_\Sigma \ll \norm{v}_2$. Thus, $\tilde D := I_n$ does not satisfy the conditions of \Cref{lemma:outlier-diag-construction}. Moreover, if any $t$ diagonal entries of $\tilde D$ are set to zero, the vector $\tilde D^{1/2} v$ will still be $O(1)$-quantitatively sparse, and $\lVert\tilde D^{1/2} v\rVert_2 \geq 2^{-t}$ whereas $\norm{v}_\Sigma = \epsilon$. So the conditions can't be satisfied unless $t \geq O(\log(1/\epsilon))$.
\end{example}

In particular, \Cref{example:decay-vec} shows that if $\tilde D$ is required to be binary, then we cannot avoid dependence on $\log(\lambda_n/\lambda_1)$ in the bound on $|\{i: \tilde D_{ii} \neq 1\}|$, which will show up in the final sample complexity bound for \rl{}. The workaround for \Cref{example:decay-vec} is to set $\tilde D_{ii} = 2^{i-k}$ for each $1 \leq i \leq k$: at this point, $\tilde D^{1/2} v$ is $\Omega(k)$-quantitatively dense, so $v$ no longer violates the guarantee of \Cref{lemma:outlier-diag-construction}. More generally, the idea is to iteratively rescale large coordinates (of each approximate dependency vector, like $v$ above) by a constant factor, rather than zeroing them out. This discourages quantitative sparsity, as formalized in \Cref{lemma:l2linfty-evolution}: after enough iterations, any vector will either be quantitatively dense or have small $\ell_2$ norm (with respect to the rescaling). We bound the number of rescaled coordinates using \Cref{lemma:heavy-coordinate-bound}.


Having discussed the high-level plan, we now proceed to the formal proof of \Cref{lemma:outlier-diag-construction}.


\begin{namedproof}{\Cref{lemma:outlier-diag-construction}}
Let $\Sigma = \sum_{i=1}^n\lambda_i u_i u_i^\t$ be an eigendecomposition of $\Sigma$ and let $P = \sum_{i=d+1}^n u_i u_i^\t$. Set $T = \lceil \log_2(\lambda_n/\lambda_1)\rceil$. We will iteratively define diagonal matrices $D^{(1)}, D^{(2)},\dots,D^{(T+1)}$ and set $\tilde D := D^{(T+1)}$. In particular, set $D^{(1)} := I_n$. For each $1 \leq t \leq T$, define the set
\begin{equation} \CS^{(t)} := \left\{i \in [n]: \sup_{x \in \ker(P) \setminus \{0\}} \frac{(D^{(t)})^{1/2}_{ii}x_i}{\norm{(D^{(t)})^{1/2} x}_2} > \frac{1}{8k}\right\}.\label{eq:st-def}\end{equation}
Then we let $D^{(t+1)}$ be the $n \times n$ diagonal matrix defined by 
\[ D^{(t+1)}_{ii} := \begin{cases} D^{(t)}_{ii}/2 & \text{ if } i \in \CS^{(t)} \\ D^{(t)}_{ii} & \text{ otherwise } \end{cases}.\]
This defines $\tilde D := D^{(T+1)}$. From the definition it is clear that $0 \prec \tilde D \preceq I_n$. We start by bounding $|\{i: \tilde D_{ii} \neq 1\}|$, which is exactly $|\bigcup_{t=1}^T \CS^{(t)}|$. By applying \Cref{lemma:heavy-coordinate-bound} to each set $\CS^{(t)}$ individually, we could get a straightforward bound of $O(dk^2 T)$. But we would like to avoid the factor of $T$, which we can do as follows. Define $V := \{\tilde D^{1/2}x: x \in \ker(P)\}$ and define
\[ \CS := \left\{i \in [n]: \sup_{y \in V\setminus\{0\}} \frac{y_i}{\norm{y}_2} > \frac{1}{8\sqrt{2}k}\right\}.\]
We claim that $\bigcup_{t=1}^T \CS^{(t)} \subseteq \CS$. Indeed, for any $i \in \bigcup_{t=1}^T \CS^{(t)}$, let $f(i) := \argmax\{t\in[T]: i \in \CS^{(t)}\}$. Then $\tilde D_{ii} = D^{(f(i))}_{ii}/2$, and $\tilde D_{jj} \leq D^{(f(i))}_{jj}$ for all $j \in [n]$. Also, since $i \in \CS^{(f(i))}$, there is some $x \in \ker(P)\setminus\{0\}$ such that $(D^{(f(i))})_{ii}^{1/2} x_i > \norm{(D^{(f(i))})^{1/2}x}_2/(8k)$. It follows that
\[ \tilde D_{ii}^{1/2} x_i = \frac{1}{\sqrt{2}} (D^{f(i)})_{ii}^{1/2} x_i > \frac{\norm{(D^{(f(i))})^{1/2}x}_2}{8\sqrt{2} k} \geq \frac{\norm{\tilde D^{1/2}x}_2}{8\sqrt{2}k}\]
and thus $y = \tilde D^{1/2} x \in V\setminus\{0\}$ satisfies $y_i/\norm{y}_2 > 1/(8\sqrt{2}k)$, so $i \in \CS$. As claimed, we get $\bigcup_{t=1}^T \CS^{(t)} \subseteq \CS$. But now since $V$ is a $d$-dimensional subspace, \Cref{lemma:heavy-coordinate-bound} gives that $|\CS| \leq 128dk^2$. This proves the first part of the lemma.


Next, fix any $v \in \RR^n$ with $\norm{\tilde D^{1/2} v}_2 > 4k\lambda_{d+1}^{-1/2}\norm{v}_\Sigma$. Since $\lambda_{d+1} P \preceq \Sigma$, it follows that $\norm{\tilde D^{1/2} v}_2 > 4k\norm{v}_P$. For any $t \in [T]$, we have $D^{(t)} \succeq \tilde  D$ and thus \begin{equation} \norm{(D^{(t)})^{1/2} v}_2 > 4k\norm{v}_P. \label{eq:proj-dist}\end{equation}

We claim that the vectors $(D^{(1)})^{1/2}v, \dots, (D^{(T+1)})^{1/2} v$ evolve according to the procedure described in \Cref{lemma:l2linfty-evolution}. Indeed, fix $t \in [T]$ and suppose that $\norm{(D^{(t)})^{1/2} v}_2 \leq k\norm{(D^{(t)})^{1/2} v}_\infty$. Pick any $i \in [n]$ such that $|(D^{(t)})^{1/2}_{ii}v_i| \geq \norm{(D^{(t)})^{1/2} v}_\infty/2$. We need to show that $i \in \CS^{(t)}$.

Define $v = a + b$ where $b \in \ker(P)$ and $a \in \vspan(P)$, so that $\norm{a}_2 = \norm{v}_P$. Then
\begin{align*}
|(D^{(t)})^{1/2}_{ii} b_i|
&\geq |(D^{(t)})^{1/2}_{ii} v_i| - |(D^{(t)})^{1/2}_{ii} a_i| \\
&\geq \frac{1}{2}\norm{(D^{(t)})^{1/2} v}_\infty - \norm{(D^{(t)})^{1/2} a}_\infty \\
&\geq \frac{1}{2k}\norm{(D^{(t)})^{1/2} v}_2 - \norm{(D^{(t)})^{1/2} a}_2 \\
&\geq \frac{1}{2k}\norm{(D^{(t)})^{1/2} v}_2 - \norm{a}_2 \\
&\geq \frac{1}{4k}\norm{(D^{(t)})^{1/2}v}_2
\end{align*}
where the first inequality is by reverse triangle inequality, the second inequality uses the choice of $i$, the fourth inequality uses that $D^{(t)} \preceq D^{(1)} = I_n$, and the fifth inequality uses \Cref{eq:proj-dist} together with the fact that $\norm{a}_2 = \norm{v}_P$. But now
\[ \norm{(D^{(t)})^{1/2} b}_2 \leq \norm{(D^{(t)})^{1/2} v}_2 + \norm{(D^{(t)})^{1/2} a}_2 \leq \norm{(D^{(t)})^{1/2} v}_2 + \norm{a}_2  \leq 2\norm{(D^{(t)})^{1/2} v}_2\]
again using the triangle inequality, the fact that $D^{(t)} \preceq I_n$, and \Cref{eq:proj-dist}. Combining the above two displays, we get that $x = (D^{(t)})^{1/2} b$ satisfies $|x_i| \geq \norm{x}_2/(8k)$. Since $b \in \ker(P)$ (and $b \neq 0$) we conclude by definition of $\CS^{(t)}$ that $i \in \CS^{(t)}$. Thus, the vectors $(D^{(1)})^{1/2}v, \dots, (D^{(T+1)})^{1/2} v$ indeed evolve according to the procedure described in \Cref{lemma:l2linfty-evolution}. Since $D^{(T+1)} = \tilde D$ and $D^{(1)} = I_n$, it follows from that lemma that at least one of the following occurs:
\begin{itemize}
    \item $\norm{\tilde D^{1/2}v}_2 > (k/2) \norm{\tilde D^{1/2}v}_\infty$, or
    \item $\norm{\tilde D^{1/2} v}_2 \leq 2^{-T} k \norm{(D^{(1)})^{1/2} v}_\infty \leq 2^{-T} k \norm{v}_2$.
\end{itemize}
In the former case, we are done. In the latter case, by choice of $T$ and the fact that $\Sigma \succeq \lambda_1 I_n$, we have $\norm{\tilde D^{1/2}v}_2 \leq \sqrt{\frac{\lambda_1}{\lambda_n}} k \norm{v}_2 \leq \frac{k}{\sqrt{\lambda_n}} \norm{v}_\Sigma \leq \frac{k}{\sqrt{\lambda_{d+1}}} \norm{v}_\Sigma$ which contradicts the assumption we made about $v$. We conclude that for any $v \in \RR^n$, either $\norm{\tilde D^{1/2} v}_2 \leq 4k\lambda_{d+1}^{-1/2}\norm{v}_\Sigma$ or $\norm{\tilde D^{1/2}v}_2 > (k/2)\norm{\tilde D^{1/2}v}_\infty$.
\end{namedproof}

\begin{lemma}\label{lemma:l2linfty-evolution}
Let $v \in \RR^n$ be an arbitrary vector with $\ell_\infty$ norm at most $1$. Consider the following procedure that we repeat $T$ times:
\begin{enumerate}[label=(\alph*)]
    \item If $\norm{v}_2 \leq k\norm{v}_\infty$, then let $S \subseteq [n]$ be a set of indices containing $\{i \in [n]: |v_i| \geq \norm{v}_\infty/2\}$. Halve $v_i$ for all $i \in S$.
    \item Otherwise, let $S \subseteq [n]$ be arbitrary. Halve $v_i$ for all $i \in S$.
\end{enumerate}
After this procedure, the final value of $v$ satisfies either $\norm{v}_2 > (k/2)\norm{v}_\infty$ or $\norm{v}_2 \leq 2^{-T}k$.
\end{lemma}

\begin{proof}
At any step where case (a) occurs, note that $\norm{v}_2$ decreases by a factor of at most $2$, whereas $\norm{v}_\infty$ decreases by a factor of at least $2$, so the ratio $\norm{v}_2/\norm{v}_\infty$ cannot decrease. Moreover, when case (b) occurs, $\norm{v}_2$ decreases by a factor of at most $2$, and $\norm{v}_\infty$ is non-increasing, so the ratio $\norm{v}_2/\norm{v}_\infty$ can decrease by at most a factor of $2$. Thus, if at any step we have $\norm{v}_2 > k\norm{v}_\infty$, inductively we have at all subsequent steps (and in particular after the final step) that $\norm{v}_2 > (k/2) \norm{v}_\infty$.

It remains to consider the case that $\norm{v}_2 \leq k\norm{v}_\infty$ at all steps. Then $\norm{v}_\infty$ decreases by a factor of at least $2$ at every step. Since initially we had $\norm{v}_\infty \leq 1$, at the end we must have $\norm{v}_\infty \leq 2^{-T}$ and thus $\norm{v}_2 \leq 2^{-T}k$.
\end{proof}

Above, we used the following simple bound on the number of basis vectors that can be correlated with a low-dimensional subspace.

\begin{lemma}[\cite{kelner2023feature}]\label{lemma:heavy-coordinate-bound}
Let $V \subseteq \RR^n$ be a subspace with $d := \dim V$. For some $\alpha > 0$ define $$S = \left\{ i \in [n]: \sup_{x \in V \setminus \{0\}} \frac{x_i}{\norm{x}_2} \geq \alpha\right\}.$$
Then $|S| \leq d/\alpha^2$. 
\end{lemma}

We now use \Cref{lemma:outlier-diag-construction} to prove that any covariance matrix $\Sigma$ with few outlier eigenvalues is rescalable, in the following quantitative sense.

\begin{lemma}\label{lemma:outlier-rescalability}
There is a constant $C_{\ref{lemma:outlier-rescalability}}$ with the following property. Let $n,k,\dlow,\dhigh \in \NN$ and let $\Sigma \in \RR^{n\times n}$ be a positive semi-definite matrix with eigenvalues $\lambda_1 \leq \dots \leq \lambda_n$. Then $\Sigma$ is $(\alpha, h)$-rescalable at sparsity $k$, where $\alpha := C_{\ref{lemma:outlier-rescalability}} k^2 \frac{\lambda_{n-\dhigh}}{\lambda_{\dlow+1}}$ and $h := C_{\ref{lemma:outlier-rescalability}} \dlow k^2 + \dhigh$. Moreover, the diagonal matrix $D$ realizing the rescaling satisfies $\max_i \frac{\Sigma_{ii}}{D_{ii}} \leq C_{\ref{lemma:outlier-rescalability}}k^2\frac{\lambda_n^2}{\lambda_1^2}$.
\end{lemma}

As previously discussed, the diagonal matrix $D$ witnessing the rescalability is an appropriate scalar multiple of the matrix $\tilde D$ constructed in \Cref{lemma:outlier-diag-construction}. The spectral lower bound needed for rescalability follows from the second guarantee of \Cref{lemma:outlier-diag-construction}. Proving the spectral upper bound requires choosing the low-rank matrix $L$ to handle both the coordinates $i$ for which $\tilde D_{ii} \neq 1$ (of which there are not many, by the first guarantee of \Cref{lemma:outlier-diag-construction}), as well as the ``high'' end of the eigenspectrum of $\Sigma$.

\begin{proof}
Let $\Sigma = \sum_{i=1}^n \lambda_i u_i u_i^\t$ be an eigendecomposition of $\Sigma$. Let $\tilde D$ be the matrix guaranteed by \Cref{lemma:outlier-diag-construction} with parameters $\dlow$ and $64k$, and define $D := \frac{\lambda_{\dlow + 1}}{2^{16}k^2} \tilde D$. By construction it is clear that $\min_i \tilde D_{ii} \geq \lambda_1/(2\lambda_n)$, so $\max_i \frac{\Sigma_{ii}}{D_{ii}} \leq \frac{2^{17}k^2 \lambda_n^2}{\lambda_1^2}$. Also define \[L := C_{\ref{lemma:outlier-rescalability}} k^2 \frac{\lambda_{n-\dhigh}}{\lambda_{\dlow+1}} \cdot (\tilde D^{-1} - I_n) + \sum_{i=\dhigh+1}^n \lambda_i u_i u_i^\t.\]
By the first guarantee of \Cref{lemma:outlier-diag-construction} we have $|\{i: \tilde D_{ii} \neq 1\}| \leq O(\dlow k^2)$, and thus $\vrank(L) \leq O(\dlow k^2) + \dhigh$. It remains to check that
\[I_n \preceq_{\MC_n(32k)} D^{-1/2}\Sigma D^{-1/2} \preceq C_{\ref{lemma:outlier-rescalability}} k^2 \frac{\lambda_{n-\dhigh}}{\lambda_{\dlow+1}} I_n + L\]
when $C_{\ref{lemma:outlier-rescalability}}$ is a sufficiently large constant. Pick any $w \in \MC_n(32k)$ and set $v := \tilde D^{-1/2} w$. We know that $\norm{\tilde D^{1/2} v}_2 \leq \norm{\tilde D^{1/2} v}_1 \leq 32k \norm{\tilde D^{1/2} v}_\infty$. So by the second guarantee of \Cref{lemma:outlier-diag-construction}, we have $\norm{\tilde D^{1/2} v}_2 \leq 256k\lambda_{\dlow+1}^{-1/2} \norm{v}_\Sigma$ (recall that we are taking the parameter $k$ in \Cref{lemma:outlier-diag-construction} to be $64k$). Hence,
\[\norm{w}_2^2 = \norm{\tilde D^{1/2} v}_2^2 \leq 2^{16}k^2 \lambda_{\dlow+1}^{-1} \cdot w^\t \tilde{D}^{-1/2} \Sigma \tilde{D}^{-1/2} w = w^\t D^{-1/2} \Sigma D^{-1/2} w.\]
This proves the first inequality. To prove the second inequality, note that
\begin{align*}
C_{\ref{lemma:outlier-rescalability}} k^2 \frac{\lambda_{n-\dhigh}}{\lambda_{\dlow+1}} D + D^{1/2} L D^{1/2}
&= C_{\ref{lemma:outlier-rescalability}} k^2 \frac{\lambda_{n-\dhigh}}{\lambda_{\dlow+1}} D^{1/2} \tilde D^{-1} D^{1/2} + \sum_{i=\dhigh+1}^n \lambda_i u_i u_i^\t \\ 
&\succeq \lambda_{n-\dhigh} I_n + \sum_{i=\dhigh+1}^n \lambda_i u_i u_i^\t \\ 
&\succeq \Sigma
\end{align*}
so long as $C_{\ref{lemma:outlier-rescalability}} \geq 2^{16}$. Applying $D^{-1/2}$ on the left and right yields the second claimed inequality.
\end{proof}

\begin{namedproof}{\Cref{theorem:outlier-lasso}}
Immediate from \Cref{theorem:rescaled-lasso} and \Cref{lemma:outlier-rescalability}.
\end{namedproof}

\section{Hardness evidence via sparse PCA with a near-critical negative spike}\label{sec:ldlr}

\input{ldlr}

\ifarxiv
\bibliographystyle{amsalpha}
\bibliography{bib}

\appendix
\fi

\section{Testing between an empty and non-empty GGM}\label{sec:testing}
In this section, we give a polynomial-time algorithm for testing between an empty and sparse non-empty Gaussian Graphical Model (GGM). Interestingly, this is possible to do with polynomial dependence on the sparsity and strength of the strongest edge in the graphical model, even though it is not known how to \emph{learn} the entire graphical structure in the same setting in polynomial time (see e.g. discussion in \cite{anandkumar2012high,misra2017information,kelner2020learning}). The sample complexity of this algorithm is suboptimal information-theoretically. In particular, it has a quadratic dependence on the sparsity $k$ even though information-theoretically, it is possible to learn the entire model with sample complexity only linear in $k$. However, our lower bound based on negatively spiked sparse PCA (see \Cref{sec:ldlr} and \Cref{remark:ggms} for the connection with GGMs) suggests that it is optimal among polynomial-time algorithms. 

The test that we use is very simple --- we simply estimate all of the correlation coefficients between the variables in our model and check if any of them are significantly different from zero.
The fact that such a test is possible to construct was alluded to in Remark 10 of \cite{kelner2020learning} without a proof or precise statement of the sample complexity, both of which we provide here. In particular, we show here that the test obtains the conjecturally sharp quadratic dependence on $k$ among efficient algorithms.
\begin{lemma}\label{lem:cond-variance-ggm}
Let $\Sigma \in \RR^{n\times n}$ be positive-definite and let $\Theta := \Sigma^{-1}$ be the corresponding precision matrix. Let $X \sim N(0,\Sigma)$. For any indices $i \ne j$, we have
\[ \Theta_{ii}\Var(X_i \mid X_{\sim i,j}) = \frac{1}{1 - \Theta_{ij}^2/\Theta_{ii}\Theta_{jj}}. \]
\end{lemma}
\begin{proof}
Define $S = \{i,j\}$.
Conditional on any fixing of $X_{\sim S} = x_{\sim S}$, the conditional density of $X_{S}$ is proportional to $\exp(-\langle x_{S}, \Theta_{SS} x_S \rangle/2 + \langle h, x_S \rangle)$ where $h$, the coefficient of the linear term, is determined by $x_{\sim S}$. This is the pdf of a Gaussian distribution with precision matrix $\Theta_{SS}$, so using the explicit formula for $2 \times 2$ matrix inversion yields
\[ \Var(X_i \mid X_{\sim S}) = (\Theta_{SS})^{-1}_{ii} = \frac{\Theta_{jj}}{\Theta_{ii} \Theta_{jj} - \Theta_{ij}^2} \]
which is equivalent to the claim. 
\end{proof}
The following crucial lemma says that if we invert a sparse matrix containing a nonnegligible off-diagonal entry, then its inverse contains a nonnegligible off-diagonal entry as well. 
\begin{lemma}\label{lem:invlb}
Let $\Theta \in \mathbb{R}^{n \times n}$ be a positive-definite matrix with at most $k + 1$ nonzero entries in each row, and define $\Sigma := \Theta^{-1}$. Then there exist indices $i, \ell \in [n]$ such that $i \neq \ell$ and
\[ \frac{\Sigma_{i \ell}}{\sqrt{\Sigma_{ii} \Sigma_{\ell \ell}}} \ge \frac{1}{2k} \max_{a \ne b} \frac{\Theta_{ab}^2}{\Theta_{aa}\Theta_{bb}}. \]
\end{lemma}
\begin{proof}
Since the statement of the lemma is invariant to rescaling (i.e. replacing $\Theta$ by $D\Theta D$ for any positive-definite diagonal matrix $D$), we may assume without loss of generality that $\Theta_{ii} = 1$ for all $i \in [n]$.
Since $\Theta$ is positive definite, it follows that $\Theta_{ij}^2 < \Theta_{ii}\Theta_{jj} = 1$  for any indices $i \ne j$. Furthermore, by Lemma~\ref{lem:cond-variance-ggm} and the law of total variance (where we define $X \sim N(0,\Sigma)$), we have
\begin{equation}\label{eqn:thetalbtosigma}
\Sigma_{ii} \ge \mathbb E \Var(X_i \mid X_{\sim i,j}) = \frac{1}{1 - \Theta_{ij}^2}. 
\end{equation}
Since $I_n = \Sigma \Theta$, we have for any coordinate $i\in [n]$ that
\[ 1 = \sum_{\ell=1}^n \Theta_{i\ell} \Sigma_{i\ell} = \Sigma_{ii} + \sum_{\ell \ne i} \Theta_{i\ell} \Sigma_{i\ell},  \]
and thus
\[ 
\Sigma_{ii} - 1 = -\sum_{\ell \ne i}\Theta_{i\ell} \Sigma_{i\ell} \le \|\Theta_{i,\sim i}\|_1 \max_{\ell} |\Sigma_{i \ell}| \le k \max_{\ell \neq i} |\Sigma_{i \ell}|\]
where the final inequality uses that $\Theta_{i,\sim i}$ has at most $k$ nonzero entries, each with magnitude at most one. If we define $\ell^\st(i) := \argmax_{\ell \neq i} |\Sigma_{i\ell}|$, then the above display implies that
\begin{equation} |\Sigma_{i \ell^\st(i) }| \ge \frac{\Sigma_{ii} - 1}{k}. \label{eq:big-corr}\end{equation}
We now consider two cases: 
\begin{enumerate}
    \item If $\max_i \Sigma_{ii} > 2$, define $i^\st := \arg\max \Sigma_{ii}$.  Then by \eqref{eq:big-corr} and choice of $i^\st$, we have
    \[ |\Sigma_{i^\st \ell^\st(i^\st)}| \ge \frac{\Sigma_{i^\st i^\st}}{2k} = \frac{\max\{\Sigma_{i^\st i^\st} , \Sigma_{\ell^\st(i^\st) \ell^\st(i^\st)}\}}{2k}. \] 
    \item Otherwise we have $\max_i \Sigma_{ii} \le 2$, so by \eqref{eq:big-corr} and \eqref{eqn:thetalbtosigma} we have for any $i,j \in [n]$ with $i \neq j$ that
    \[ |\Sigma_{i \ell^\st(i) }| \ge \frac{\Sigma_{ii} - 1}{k} \ge \frac{1/(1 - \Theta_{ij}^2) - 1}{k} \ge \frac{1/(1 - \Theta_{ij}^2) - 1}{2k} \max\{\Sigma_{ii},\Sigma_{\ell^\st  \ell^\st }\}. \]
    In particular, this bound holds when $i$ and $j$ are chosen to maximize $|\Theta_{ij}|$.
\end{enumerate}
Therefore in either case, there exists some $i \in [n]$ such that
\[ \frac{\Sigma_{i \ell^\st(i)}}{\max \{\Sigma_{ii}, \Sigma_{\ell^\st(i)  \ell^\st(i) }\}} \ge \frac{1}{2k} \min \left\{1, \frac{1}{1 - \max_{a \ne b} \Theta_{ab}^2} - 1 \right\}. \]
Finally, using the inequality $1/(1 - x) - 1 \ge x$ which is valid for all $x < 1$, we obtain
\[ \frac{\Sigma_{i \ell^\st(i) }}{\max \{\Sigma_{ii}, \Sigma_{\ell^\st(i)  \ell^\st(i) }\}} \ge \frac{1}{2k} \min\left\{1, \max_{a \ne b} \Theta_{ab}^2 \right\} = \frac{1}{2k} \max_{a \ne b} \Theta_{ab}^2.\]
Using the fact that $\max \{\Sigma_{ii}, \Sigma_{\ell^\st(i)  \ell^\st(i) } \} \ge \sqrt{\Sigma_{ii} \Sigma_{\ell^\st(i)  \ell^\st(i) }}$ proves the result. 
\end{proof}
This structural result yields a tester because we can directly estimate correlations from data. Very precise results about the sample correlation coefficient were obtained by \cite{hotelling1953new}. Below, we give a simple argument which yields easy-to-use nonasymptotic bounds. 

We recall the following basic fact about Gaussians. See e.g. Lemma 2 of \cite{kelner2020learning} for an explicit proof.
\begin{lemma}[classical]
\label{lemma:var-cov-exp}
If $X,Y$ are jointly Gaussian random variables with $\Var(Y)>0$, then
\[ \Var(Y) - \Var(Y \mid X) = \frac{\operatorname{Cov}(X,Y)^2}{\Var(Y)}. \]
\end{lemma}

For convenience, we use the following lemma. It is a special case of a general statement about testing for changes in conditional variance, which is closely related to classical results about non-central F-statistics and Wishart matrices (see e.g. \cite{keener2010theoretical}).
\begin{lemma}[Special case of Lemma 12 of \cite{kelner2020learning}]\label{lem:change-variance}
For jointly Gaussian random variables $X,Y$ with $\Var(Y|X)>0$, define
\[ \gamma(Y;X) := \frac{\Var(Y) - \Var(Y \mid X)}{\Var(Y \mid X)}. \]
There exists an efficiently computable (i.e. polynomial time) statistic $\hat \gamma$ of $m$ i.i.d. copies $(X_1,Y_1),\ldots,(X_m,Y_m)$ of $(X,Y)$ such that
\[ \left|\sqrt{\hat \gamma} - \sqrt{\gamma}\right| \le \sqrt{\frac{4\log(4/\delta)}{m}} + \sqrt{\gamma/64}. \]
\end{lemma}
\begin{theorem}
Suppose that $X \sim N(0,\Sigma)$, $\kappa \ge 0, \delta > 0$ and consider the following two hypotheses. The null hypothesis $H_0$ is that $\Sigma$ is a diagonal matrix. The alternative hypothesis $H_1$ is that $\Sigma = \Theta^{-1}$ where $\Theta$ has $(k + 1$)-sparse rows, and where
\[ \kappa \le \max_{a \ne b} \frac{|\Theta_{ab}|}{\sqrt{\Theta_{aa} \Theta_{bb}}}, \]
i.e.\ the maximal partial correlation coefficient is at least $\kappa$. Then provided $m = \Omega(k^2\log(n/\delta)/\kappa^4)$ i.i.d. copies of $X$, we can distinguish in polynomial time between $H_0$ and $H_1$ with sum of probability of type I and type II errors at most $\delta$.
\end{theorem}
\begin{proof}
We test the maximal correlation of $X_i$ and $X_j$ over all $i,j \in [n]$ with $i \neq j$.
Under the null hypothesis the true correlation $\gamma(X_i;X_j)$ equals zero no matter the choice of $i,j$. Under the alternative hypothesis, let $i,j$ be the indices given by \Cref{lem:invlb} and without loss of generality suppose $\Var(X_i) \leq \Var(X_j)$. Then $\gamma = \gamma(X_i;X_j)$ can be lower bounded as
\[ \gamma = \frac{\Var(X_i) - \Var(X_i \mid X_j)}{\Var(X_i|X_j)} \geq \frac{\operatorname{Cov}(X_i,X_j)^2}{\Var(X_i)^2} \ge \frac{\operatorname{Cov}(X_i,X_j)^2}{\Var(X_i) \Var(X_j)} \ge \left(\frac{1}{2k} \max_{a \ne b} \frac{\Theta_{ab}^2}{\Theta_{aa} \Theta_{bb}}\right)^2\]
where the first inequality uses \Cref{lemma:var-cov-exp} and the fact that $\Var(X_i|X_j) \leq \Var(X_i)$. Thus, by the theorem assumption, it holds that $\gamma(X_i;X_j) \geq \kappa^4/(4k^2)$. By Lemma~\ref{lem:change-variance} and a union bound over all choices of $i,j$, given 
\[ m = \Omega(k^2\log(4n/\delta)/\kappa^4) \]
samples we can distinguish between the two hypotheses with the sum of probability of type I and type II error at most $\delta$. 
\end{proof}
\section{Simulation}\label{sec:simulation}
\begin{figure}[t]
    \centering
    \includegraphics[width=\textwidth]{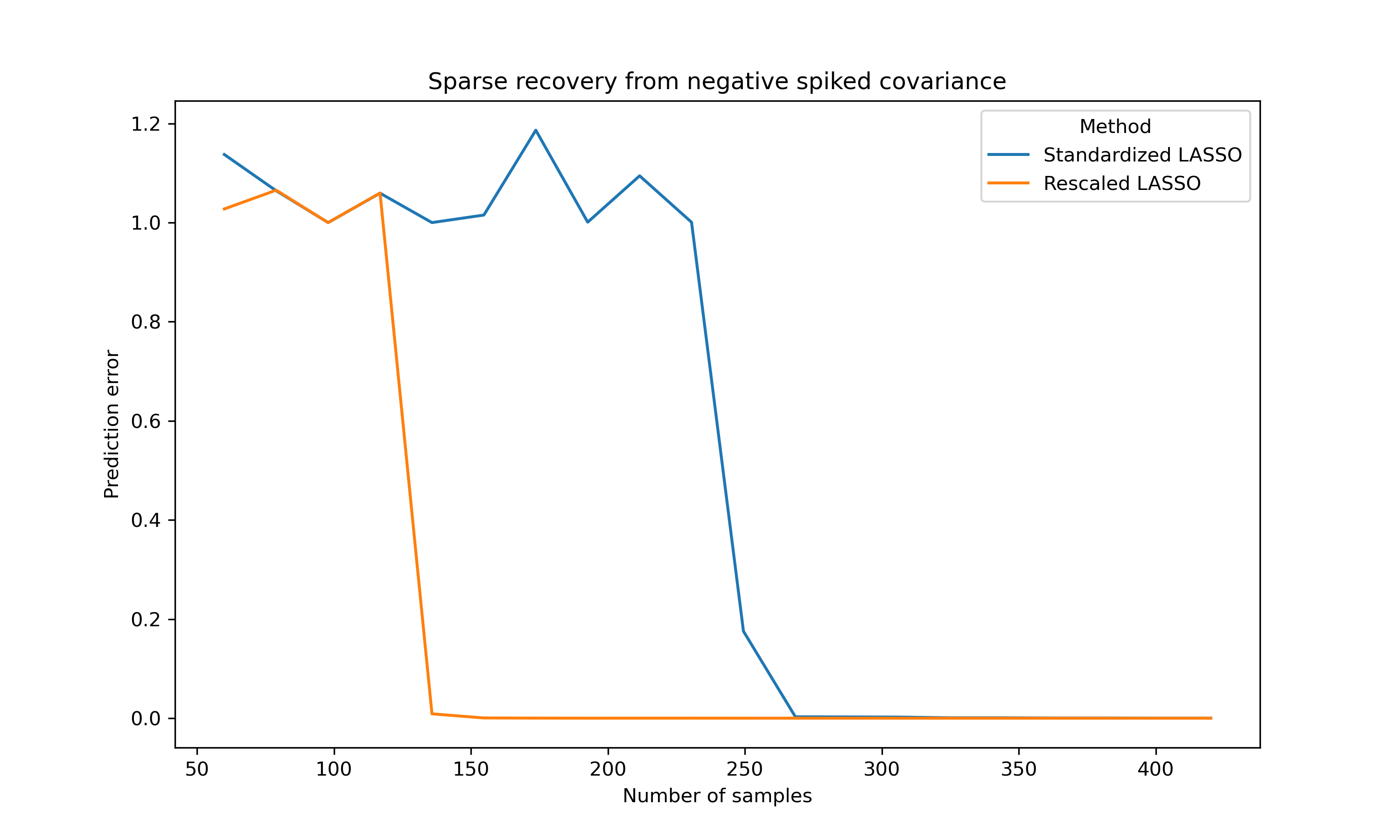}
    \caption{Standardized Lasso vs \texttt{RescaledLasso} in a simple example with varying number of samples. For each datapoint, the covariates were drawn i.i.d. from the negatively spiked sparse PCA model with ambient dimension $n = 300$ and with $\beta = -0.99$. For covariate vector $X$, the ground truth response $Y$ is generated as 
    $Y = \frac{1}{\sqrt{(1 + \beta) k}} \langle 1_S, X \rangle$ 
    where $S$ is the set of coordinates of size $k = 5$ where the spike is supported. As we expect from the theory, \texttt{RescaledLasso} recovers the signal from fewer samples than Lasso applied with the usual standardization/normalization of covariates.}
    \label{fig:simulation}
\end{figure}
As a simple example, we ran \rl{} in a simple simulation on synthetic data  where covariates follow negatively spiked sparse PCA (Figure~\ref{fig:simulation}) and verified that it indeed had improved prediction error compared to usual standardization of covariates. We used the glmnet package in R and optimized the regularization hyperparameter using a validation set. 

When we ran the algorithm in the simulation, we changed the value of the hyperparameters  $\textsf{DIV}$ to $1.1$ and of $\textsf{B}$ to $2k$ --- generally speaking, we expect setting the value of $\textsf{DIV}$ closer to $1$ will not significantly hurt the statistical performance (and may help a bit in some cases) although it may make the algorithm run a bit slower, and the proof does not depend on the particular value of $\textsf{DIV}$ chosen in the original pseudocode (which was only chosen for mathematical convenience). Also, instead of updating the scale only of $i_{min}$, in each pass we we update the scale of all indices $i$ such that $\frac{1}{m} \|\mathbb X v^{(t,i)}\|_2^2 \le 1$ --- this significantly reduces the number of iterations and thus the runtime of the algorithm (it can be checked that the theoretical guarantees also hold with this modification).

\section{Technical lemmas}

\begin{lemma}\label{lemma:chisq-conc}
Let $\Sigma \in \RR^{n\times n}$ be a positive semi-definite matrix and fix $w \in \RR^n$. Let $\BX \in \RR^{m\times n}$ have i.i.d. rows $X_1,\dots,X_m\sim N(0,\Sigma)$. If $m \geq 32\log(2/\delta)$, then
\[\Pr\left[\frac{1}{2} \norm{w}_\Sigma^2 \leq \frac{1}{m}\norm{Xw}_2^2 \leq 2\norm{w}_\Sigma^2\right] \geq 1-\delta.\] 
\end{lemma}

\begin{proof}
By a change of variables, it suffices to consider the case $\Sigma = I_n$. But then $\norm{\BX w}_2^2/\norm{w}_2^2$ is distributed as a $\chi$-squared random variable with $m$ degrees of freedom. The statement follows from concentration of $\chi$-squared random variables.
\end{proof}

\begin{lemma}[Hanson-Wright inequality \citep{rudelson2013hanson}]\label{lemma:hw}
Let $\Sigma\in\RR^{n\times n}$ be a positive semi-definite matrix. Let $X \sim N(0,\Sigma)$. Then for any $t>0$,
\[\Pr[|\norm{X}_2^2 - \tr(\Sigma)| > t] \leq 2\exp\left(-c\min\left(\frac{t^2}{\norm{\Sigma}_F^2}, \frac{t}{\norm{\Sigma}_\mathsf{op}}\right)\right)\]
where $c>0$ is a universal constant.
\end{lemma}

In particular we will use the following simplification of the Hanson-Wright inequality.

\begin{corollary}\label{cor:hw}
Let $\Sigma\in\RR^{n\times n}$ be a positive semi-definite matrix. Let $X \sim N(0,\Sigma)$. Let $\delta \in (0,1/4)$. Then
\[ \Pr[\norm{X}_2^2 > C\tr(\Sigma)\log(2/\delta)] \leq \delta\]
where $C>0$ is a universal constant.
\end{corollary}

\begin{proof}
Observe that $\norm{\Sigma}_\mathsf{op} \leq \norm{\Sigma}_F \leq \tr(\Sigma)$ (since $\norm{\Sigma}_{\mathsf{op}}$ is the $\ell_\infty$ norm of the eigenvalues of $\Sigma$, whereas $\norm{\Sigma}_F$ is the $\ell_2$ norm and $\tr(\Sigma)$ is the $\ell_1$ norm). Thus, \Cref{lemma:hw} gives that for any $t>0$,
\[\Pr[|\norm{X}_2^2 - \tr(\Sigma)| > t] \leq 2\exp\left(-c\min\left(\frac{t^2}{\tr(\Sigma)^2}, \frac{t}{\tr(\Sigma)}\right)\right).\]
Taking $t := \max(1, 1/c)\tr(\Sigma)\log(2/\delta)$ gives the claimed result.
\end{proof}

For a symmetric matrix $A$, let $\lambda_1(A) \leq \dots \leq \lambda_n(A)$ denote the eigenvalues of $A$. The following inequality is well-known.

\begin{lemma}[Weyl's inequality]\label{lemma:weyl}
Let $N,R \in \RR^{n\times n}$ be symmetric matrices. Suppose that $\vrank(R) = r$. Then
\[\lambda_{n-r}(N+R) \leq \lambda_n(N).\]
\end{lemma}
\end{document}

%% file: overview.tex
In this section we give overviews of the proof of \Cref{theorem:rescaled-lasso-intro} (via a new variable normalization procedure) and \Cref{thm:no-subquadratic-alg-intro} (via a new connection with negative spike sparse PCA). 

Throughout the paper, we adopt the following notation. For $n\times n$ symmetric matrices $A,B$, we write $A \preceq B$ to denote that $B-A$ is positive semi-definite; for a set $S \subseteq \RR^n$, we write $A \preceq_{S} B$ to denote that $v^\t A v \leq v^\t B v$ for all $v \in S$. For a matrix $A\in \RR^{n\times n}$, we write $\diag(A)$ to denote the matrix $D \in \RR^{n\times n}$ defined by $D_{ij} = A_{ij}\mathbbm{1}[i=j]$. We write $I_n$ to denote the $n \times n$ identity matrix. For a positive semi-definite matrix $A \in \RR^{n\times n}$ and vector $v \in \RR^n$, $\norm{v}_A$ denotes $\sqrt{v^\t A v}$.

\subsection{Upper bounds}\label{sec:overview-upper}

Informally, \Cref{theorem:rescaled-lasso-intro} states that there is a computationally efficient and sample-efficient algorithm for sparse linear regression (as modelled in \Cref{def:slr-intro}) whenever the covariance matrix $\Sigma$ is rescalable (see \Cref{def:rescalable}). To reiterate, the main algorithmic difficulty is that the diagonal matrix $D$ in \Cref{def:rescalable} is unknown and potentially even unidentifiable, so we cannot simply perform the ``oracle rescaling'' $X \mapsto D^{-1/2} X$. 
In particular, in \Cref{setting:lvm}, where $\Sigma = D+L$, the diagonal entry $D_{ii}$ measures the conditional variance of the covariate $X_i$ with respect to the latent variables. Even in very simple examples, these conditional variances are unidentifiable:
\begin{example}
Consider a model with latent variable $H_1 \sim N(0,1-\epsilon^2)$, and independent covariates $X_1 = H_1 + N(0, \epsilon^2)$ and $X_2 \sim N(0, 1)$. Then $X_1$ has conditional variance $\epsilon^2$, whereas $X_2$ has conditional variance $1$. However, from the observed data it is impossible to tell which of $X_1$ or $X_2$ is connected to the latent, since either way the joint law is $X \sim N(0,I_2)$. 
\end{example}
Fortunately, we are happy with \emph{any} good rescaling matrix $\hat D$, even if it's not the oracle one.
Based on our mathematical understanding of the Lasso, a rescaling should be ``good'' if the rescaled covariates admit no (quantitatively) sparse approximate dependencies \--- and these are identifiable, even from a small number of samples. This motivates the algorithm described below.

\paragraph{Algorithm description.} The procedure \lp{} takes as input the covariate data $(X^{(j)})_{j=1}^m$ and the sparsity level $k$, and then initializes $\hat{D}$ to be the diagonal of the empirical covariance matrix $\hat\Sigma = \frac{1}{m}\sum_{j=1}^m X^{(j)} (X^{(j)})^\t$. Note that this initialization corresponds to the ``standard'' covariate normalization, which may be highly sub-optimal in the presence of strong correlations.

To fix this, the procedure iteratively decreases entries of $\hat D$ until it can certify that the resulting scaling is good. At each step, for each covariate $X_i$, the procedure solves the following program to compute how much of the (empirical) variance of $X_i$ cannot be explained by quantitatively sparse combinations of the other covariates:\footnote{We remark that this program is similar in spirit to a natural convex relaxation for detecting the sparsest vector in a subspace --- c.f.  \cite{demanet2014scaling,spielman2012exact}. See also our related work section.}
\begin{equation} \min_{v \in \RR^n: \substack{\norm{\hat D^{1/2} v}_1 \leq 16k \\ (\hat D^{1/2} v)_i = 1}} \frac{1}{m} \sum_{j=1}^m \langle X^{(j)}, v\rangle^2.\label{eq:cond-variance-overview}\end{equation}
If \eqref{eq:cond-variance-overview} is at least a constant for all $i \in [n]$, then 
\lp{} returns $\hat D$. Otherwise, the procedure picks some $i$ for which \eqref{eq:cond-variance-overview} is small, and then halves $\hat D_{ii}$ and repeats.

After \lp{} returns $\hat D$, the main algorithm \rl{} simply applies the rescaling $X^{(j)} \mapsto \hat D^{-1/2} X^{(j)}$ to each sample covariate, solves Lasso, and unscales the solution (equivalently, it 
uses a coordinatewise penalty $\lVert \hat D^{1/2} w\rVert_1$). See \Cref{alg:intro} for pseudocode.




\SetKwFunction{adabp}{AdaptedBP}

\SetKwProg{myalg}{Procedure}{}{}
\begin{algorithm2e}[t]
  \SetAlgoLined\DontPrintSemicolon
  \caption{Adaptive variable rescaling}
  \label{alg:intro}

\SetKwBlock{Repeat}{repeat}{}
\SetKw{Break}{break}
\SetAlgoNoEnd
\myalg{\lp{$\BX$, $k$}}{
$\textsf{DIV} \gets 2$; $\textsf{B} \gets 16k$; $\hat{\Sigma} \gets \frac{1}{m}\BX^\t \BX$; $\hat{D}^{(1)} \gets \diag(\hat{\Sigma})$\;
\For{$t = 1, 2, 3, \dots$}{
    For every $1 \leq i \leq n$, compute
        \begin{equation} v^{(t,i)} \gets \argmin_{v \in \RR^n: \ \substack{\norm{(\hat{D}^{(t)})^{1/2} v}_1 \leq \textsf{B} \\ ((\hat{D}^{(t)})^{1/2} v)_i = 1} }\frac{1}{m} \norm{\BX v}_2^2.
        \label{eq:v-prog}
        \end{equation}
    \hspace{-.3cm} 
    $\imin \gets \argmin_{i\in[n]} \frac{1}{m} \norm{\BX v^{(t,i)}}_2^2$\; 
    \lIf{$\frac{1}{m}\norm{\BX v^{(t,\imin)}}_2^2 \leq 1$}{
            $\hat{D}^{(t+1)} \gets \hat{D}^{(t)}$; $\hat{D}^{(t+1)}_{\imin\imin} \gets \hat{D}^{(t+1)}_{\imin\imin} / \textsf{DIV}$
    }
    \lElse{
        \Return $\hat{D}^{(t)}$
    }
}
}

\myalg{\rl{$(X^{(j)},y^{(j)})_{j=1}^m$, $k$, $\lambda$}}{
Define $\BX \in \RR^{m \times n}$ by
    $\BX \gets \begin{bmatrix} X^{(1)} & X^{(2)} & \dots & X^{(m)} \end{bmatrix}^\t.$\;
$\hat{D} \gets$ \lp{$\BX$, $k$}.\;
Compute and return $\hat w$, the solution to the modified Lasso:
\begin{equation}\hat{w} \gets \argmin_{w \in \RR^n}\frac{1}{m}\norm{\BX w - y}_2^2 + \lambda \norm{\hat{D}^{1/2} w}_1.\label{eq:rescaled-lasso-program}\end{equation}
}

\end{algorithm2e}


\paragraph{Why does this work?} The 
heart of our analysis is the following guarantee about the output $\hat D$ of $\lp{}$. It states that if the empirical covariance matrix $\hat\Sigma$ is spectrally lower bounded (on quantitatively sparse vectors) after the ``oracle rescaling'' $D^{-1/2}$, then it's also lower bounded after the estimated rescaling $\hat D^{-1/2}$, and moreover $D$ is an approximate lower bound on $\hat D$.\footnote{Note however that the guarantee with $\hat D^{-1/2}$ is qualitatively weaker than the guarantee with the oracle rescaling: in \Cref{lemma:lp-analysis-overview}, $\frac{1}{m}\lVert\BX\hat D^{-1/2} v\rVert_2^2$ is lower bounded in terms of $\norm{v}_\infty$ rather than $\norm{v}_2$. From a technical perspective, this discrepancy is the source of the quadratic dependence on $k$ in the sample complexity of \rl{}.}

\begin{lemma}\label{lemma:lp-analysis-overview}
Let $n,m,k \in \NN$. Let $\BX \in \RR^{m \times n}$. Suppose that $I_n \preceq_{\MC_n(32k)} D^{-1/2} \hat\Sigma D^{-1/2}$ where $D \succ 0$ is a diagonal matrix, and $\hat\Sigma := \frac{1}{m}\BX^\t \BX$. Then the algorithm $\hat{D} \gets \lp{\text{$\BX, k$}}$
terminates after at most $T := n\log \max_{i \in [n]} \frac{2\hat\Sigma_{ii}}{D_{ii}}$ repetitions, and moreover the output satisfies:
\begin{enumerate}
    \item $\hat{D} \succeq \frac{1}{2} D$.
    \item $\frac{1}{m}\lVert\BX \hat D^{-1/2} v\rVert_2^2 > 1$ for all $v \in \RR^n$ with $\norm{v}_\infty = 1$ and $\norm{v}_1 \leq 16k$.
\end{enumerate}
\end{lemma}
For any rescalable $\Sigma$, the spectral lower bound needed to apply \Cref{lemma:lp-analysis-overview} to samples from $N(0,\Sigma)$ is inherited (with high probability) from the assumed lower bound on $\Sigma$ (\Cref{lemma:transfer-lb}). This crucially uses a generalization bound derived from the upper bound $D^{-1/2}\Sigma D^{-1/2} \preceq \alpha I_n + L$. Next, let us explain why \rl{} has low prediction error assuming \Cref{lemma:lp-analysis-overview}. 

It is simplest to consider the special case of sparse linear regression where $\sigma = 0$ (i.e. the responses are noiseless). In this setting, instead of solving the rescaled Lasso program \eqref{eq:rescaled-lasso-program}, one would solve the rescaled basis pursuit:
\[\hat w \in \argmin_{w \in \RR^n: \BX w = y} \lVert\hat D^{1/2} w\rVert_1,\]
where $\BX\in \RR^{m\times n}$ is the matrix with rows $X^{(1)},\dots,X^{(m)}$. When does this program fail to return the true solution $w^\st$? Since $y = \BX w^\st$, the program only fails when there is some alternative $\hat w \in \RR^n$ with $\BX(\hat w - w^\st) = 0$ and $\lVert\hat D^{1/2} \hat w\rVert_1 \leq \lVert\hat D^{1/2} w^\st\rVert_1$. By a standard manipulation, the second inequality (together with $k$-sparsity of $w^\st$) implies that the error vector $e := \hat w - w^\st$ is $O(k)$-quantitatively sparse with respect to $\hat D$, i.e. \[ \lVert\hat D^{1/2} e\rVert_1 \leq O(k) \cdot \lVert\hat D^{1/2} e\rVert_\infty. \]
By the second guarantee of \Cref{lemma:lp-analysis-overview} with $v := \hat D^{1/2} e$ (and the fact that $e \neq 0$), it follows that $\frac{1}{m}\norm{\BX e}_2^2 > 0$. This is a contradiction, so in fact the rescaled basis pursuit must return $w^\st$.
Extending this argument to the general, noisy setting follows a similar rough blueprint;
we defer the details to \Cref{sec:rescaling}. 
We now sketch the proof of the key \Cref{lemma:lp-analysis-overview} (see \Cref{sec:rescaling} for the full proof).

\begin{proof}[Proof sketch for \Cref{lemma:lp-analysis-overview}]
The second guarantee is immediate from the termination condition of \lp{}. The bound on the number of repetitions in \lp{} will be immediate once we show that the output satisfies $\hat D \succeq \frac{1}{2} D$, since at every repetition, the algorithm halves at least one entry of $\hat D$. 

The only non-obvious claim (and the heart of the result) is that $\hat D \succeq \frac{1}{2} D$ at termination. For intuition, in this sketch we'll only consider the latent variable model setting (i.e. $\Sigma = D+L$) and the large sample limit $\hat \Sigma \approx \Sigma$, but the proof generalizes. Say that each covariate has variance $1$, so the algorithm has initialized $\hat D := I_n$. Since $D$ measures the conditional variances of the covariates, it's clear that $D \preceq \hat D$ holds initially. Now suppose there is some vector $v$ with $v_i = 1$ and $\norm{v}_\Sigma \ll 1$. Then $v$ describes an approximate dependency involving covariate $X_i$, so the conditional variance of $X_i$ (which is exactly $D_{ii}$) \emph{must} be small: formally, \[D_{ii} = D_{ii} v_i^2 \leq \norm{v}_D^2 \leq \norm{v}_\Sigma^2 \ll 1 = \hat D_{ii}.\]
Thus, the algorithm can safely set $\hat D_{ii} \gets \hat D_{ii}/2$ while preserving the invariant $\hat D \succeq \frac{1}{2} D$. At each subsequent step, similar logic applies, so at termination $\hat D \succeq \frac{1}{2} D$ still holds.
\end{proof}

\subsection{Lower bounds}\label{sec:overview-lower}

\Cref{thm:no-subquadratic-alg-intro} asserts that \rl{}, which requires only $O(k^2 \log n)$ samples to achieve prediction error $O(\sigma^2)$ whenever $\Sigma$ is $(1,k)$-rescalable (and $w^\st$ is $k$-sparse), is essentially optimal among polynomial-time algorithms, under a conjecture about the power of low-degree polynomials. We prove the theorem by studying negative-spike sparse PCA in a ``near-critical'' regime. Concretely, this refers to a distribution testing problem between a spiked Wishart distribution $\BP_{n,k,\beta,m}$ and a null distribution $\BQ_{n,m}$ defined below, in the regime where $\beta$ is negative and close to $-1$.
\begin{definition}\label{def:sparse-prior}
Let $n,k \in \NN$ with $k \leq n$. The \emph{fixed-size sparse Rademacher prior} $\CW_{n,k}$ is the distribution on $\RR^n$ where $w \sim \CW_{n,k}$ is drawn by: 1. sampling a subset $S \subseteq [n]$ of size $k$ uniformly at random, and 2. setting
$w_i \sim \Unif(\{1/\sqrt{k},-1/\sqrt{k}\})$ for each $i\in S$ and $w_i = 0$ otherwise. 
\end{definition}
\begin{definition}\label{def:pq}
Let $n,k,m \in \NN$ with $k \leq n$ and $\beta \in (-1,\infty)$. The \emph{$k$-sparse spiked Wishart distribution} $\BP_{n,k,\beta,m}$ is 
the distribution of $(Z^{(j)})_{j=1}^m$ where first we sample $w \sim \CW_{n,k}$ (\Cref{def:sparse-prior}), and then $(Z^{(j)})_{j=1}^m \sim N(0,I_n + \beta ww^\t)^{\otimes m}$. The \emph{null distribution} $\BQ_{n,m}$ is defined as $N(0, I_n)^{\otimes m}$. 
\end{definition}
For the hardness of negative-spike sparse PCA, we show that for any spike strength $\beta \in (-1,1)$, degree-$\log^{O(1)}(n)$ polynomials require sample complexity $m \geq \tilde \Omega(k^2)$ to test between $\BP_{n,k,\beta,m}$ and $\BQ_{n,m}$ (\Cref{theorem:ldlr-main}). This result largely follows similar bounds for positive-spike sparse PCA \citep{bandeira2020computational, ding2023subexponential}. To give further evidence of hardness, we also prove a lower bound for a natural semidefinite programming relaxation (\Cref{thm:sdp}) --- this is inspired by analogous results in the positive spike setting \citep{krauthgamer2015semidefinite}, although we need to use a different construction since we are minimizing, rather than maximizing, the SDP objective.

We then show that an improved algorithm for sparse linear regression (with rescalable covariances) would yield an improved tester for negative spike sparse PCA when $\beta$ is close to $-1$:
\begin{theorem}\label{theorem:pca-to-slr-overview}
Let $\mslr:\NN\times\NN\to\NN$ be a function, and suppose that there is a polynomial-time algorithm $\MA$ with the following property. For any $n,k\in\NN$, $\sigma>0$, positive semi-definite $(1,k)$-rescalable matrix $\Sigma \in \RR^{n\times n}$, $k$-sparse vector $w^\st \in \RR^n$, and $m \geq \mslr(n,k)$, the estimate $\wh \gets \MA((X^{(j)},y^{(j)})_{j=1}^m)$ satisfies
\[\Pr[\norm{\wh - w^\st}_\Sigma^2 \leq \sigma^2/10] = 1-o(1) \]
where the probability is over the randomness of $\MA$ and i.i.d.\ samples $(X^{(j)},y^{(j)})_{j=1}^m$ from $\SLR_{\Sigma,\sigma}(w^\st)$.

Then there is a polynomial-time algorithm $\MA'$ with the following property. For any $n,m,k \in \NN$ and $\beta \in (-1, -1 + 1/(2k)]$, if $m \geq \mslr(n,k) + 1600\log(n)$, then
\begin{equation} \left|\Pr_{Z\sim \BP_{n,k,\beta,m}}[\MA'(Z)=1] - \Pr_{Z\sim \BQ_{n,m}}[\MA'(Z)=1]\right| = 1 - o(1).\label{eq:detection-success-overview} \end{equation}
\end{theorem}
The idea behind the reduction is to check (using the sparse linear regression algorithm $\MA$) whether any covariate in the given sparse PCA data can be explained by the other covariates better than one would expect under the null distribution. Concretely, for a sample $Z \sim N(0, I_n + \beta ww^\t)$, for any $i$ in the support of $w$, it can be observed that $\E[Z_i \mid Z_{\sim i}]$ is a $(k-1)$-sparse linear combination of the remaining covariates, and $Z_i$ has conditional variance 
\[ \sigma^2 := \Var(Z_i \mid Z_{\sim i}) = \frac{1+\beta}{1+\beta(1-1/k)}. \] In the near-critical regime $\beta \in (-1, -1+1/(2k)]$, we have $\sigma^2 \le 1/2$. Hence, using $\MA$, we can distinguish from the null hypothesis that our samples are drawn from $N(0, I_n)$. If $\MA$ were as statistically efficient as Best Subset Selection, then we could also solve the distinguishing problem with only $O(k\log n)$ samples. See \Cref{sec:reduction} for the full proof. 
\begin{remark}\label{rmk:positive-pc-fails}
It remains unknown whether an analogue of \Cref{thm:no-subquadratic-alg-intro} can be proven under the Planted Clique Hypothesis (or any other standard average-case complexity hypothesis). One could hope to achieve such a result by reducing \emph{positive-spike} sparse PCA to sparse linear regression. But in the above reduction, if $\beta > 0$ then the conditional variance of any $i$ in the support of the spike is within $[1,1+1/\Omega(k)]$, so even using the (computationally inefficient) guarantees for best subset selection, we would need $\Omega(k^2)$ samples to distinguish from the null hypothesis (c.f. \cite{bresler2018sparse}). Only in the near-critical negative spike regime do we get a sufficiently large gap in conditional variance for the hardness reduction to go through.

Informally, the reason such a reduction fails to establish hardness is that the information-theoretically optimal algorithms for positive spike sparse PCA need to optimize simultaneously over both a \emph{sparsity} and \emph{low rank} constraint on the covariance. 
Surprisingly, when we have a near-critical negative spike, using Best Subset Selection (which only enforces sparsity and has no explicit notion of low-rank structure) actually achieves the information-theoretic threshold.
\end{remark}
\begin{remark}
The negative spike sparse PCA problem can be viewed as a real-valued analogue of the celebrated sparse parities with noise (SPN) problem. See \Cref{remark:analogy-spn} for explanation.
\end{remark}


%% file: related.tex
\subsection{Algorithms} \label{sec:related-upper}

Sparse linear regression has been widely studied throughout fields such as statistics, theoretical computer science, and signal processing, see e.g.  \citep{candes2006stable, raskutti2010restricted,donoho1989uncertainty,donoho2005stable,zhang2017optimal},
and \cite[Section 7.7]{wainwright2019high} for additional historical context.
In the random-design model (\Cref{def:slr-intro}) we consider throughout the paper, it is well-known that the Best Subset Selection estimator \citep{hocking1967selection} achieves prediction error $O(\sigma^2 k\log(n)/m)$ with high probability \citep{foster1994risk}. 
However, this requires $O(n^k)$ time, and thus is prohibitively costly even for small $k$.\footnote{The benefits of best subset selection have motivated major integer programming efforts, see e.g. \cite{bertsimas2016best, hastie2020best}; in the worst case, it likely requires $n^{(1 - o(1))k}$ time \citep{gupte2021fine}.}

Celebrated estimators based on $\ell_1$-regularization (e.g. Lasso \citep{tibshirani1996regression} and the Dantzig Selector \citep{candes2007dantzig}) and greedy variable selection (e.g. Orthogonal Matching Pursuit \citep{cai2011orthogonal}) are highly practical alternatives that can be computed in polynomial time. But the analyses of these estimators all require some additional assumption on $\Sigma$ in order to achieve optimal sample complexity. The archetypal guarantee is the following: if $\lambda_{\max}(\Sigma)/\lambda_{\min}(\Sigma) \leq \kappa$, then the Lasso program
\begin{equation} \hat w \gets \argmin_{w \in \RR^n} \frac{1}{m}\sum_{j=1}^m (\langle X^{(j)}, w\rangle - y^{(j)})^2 + \lambda \norm{w}_1\label{eq:lasso}\end{equation}
achieves prediction error $O(\sigma^2 k \kappa \log(n)/m)$ \cite[Theorems 7.16 + 7.20]{wainwright2019high} with high probability over samples $(X^{(j)},y^{(j)})_{j=1}^m \sim \SLR_{\Sigma,\sigma}(w^\st)$, for an appropriate choice of  $\lambda>0$.

Significant efforts have gone into establishing guarantees for Lasso under the weakest assumptions possible, leading to more general assumptions such as 
the Restricted Eigenvalue Condition \citep{bickel2009simultaneous} and the compatibility condition \citep{van2009conditions}; see also the submodularity condition for Orthogonal Matching Pursuit \citep{das2011submodular}. All of these assumptions boil down to some variant of a condition number bound. This is an inherent limitation of Lasso and other classical estimators, which provably fail (i.e. have poor sample complexity) in the presence of strong, sparse linear dependencies among covariates (see \Cref{sec:related-lower} for more details and references). Such dependencies may easily occur in the settings that we study in this paper.
\paragraph{Preconditioned Lasso.} Most closely related to this paper are the recent works \citep{kelner2021power, kelner2023feature}, which also identify natural structural assumptions on $\Sigma$ under which Lasso may fail, but a more clever algorithm succeeds. In particular, \cite{kelner2021power} studied the setting where the covariates are drawn from a Gaussian Graphical Model with low treewidth. In this setting, they showed that there is a \emph{preconditioned Lasso} program \--- namely, a modification of \Cref{eq:lasso} where the regularization term $\norm{w}_1$ is replaced by $\norm{S^\t w}_1$ for some matrix $S$ \--- with near-optimal statistical performance, and that, given the graphical structure (i.e. the support of the precision matrix $\Sigma^{-1}$), the preconditioner $S$ can be efficiently computed. \Cref{setting:lvm} is incomparable to the low treewidth assumption, and our algorithm does not assume any knowledge about the ground truth.

\cite{kelner2023feature} introduced the setting where the spectrum of $\Sigma$ may have a small number of outliers (which we refer to as \Cref{setting:outliers} above). Their main result is a polynomial-time algorithm that achieves prediction error $(k\lambda_{n-d}/\lambda_{d+1} + k^{O(k)}d)\log(n)/m$ with high probability, where $\lambda_1 \leq \dots \leq \lambda_n$ are the eigenvalues of $\Sigma$, and $d \in \{1,\dots,n-1\}$. Note the exponential dependence on $k$, so the previous result is essentially vacuous for $k = \Omega(\log n)$. In contrast, \Cref{theorem:outlier-lasso-intro} incurs only polynomial dependence on $k$. Additionally, unlike the prior work, our algorithm does not need to be given $\Sigma$, which is a significant advantage in applications where the sample complexity is likely sublinear in $n$, and thus, the empirical covariance is a poor approximation for $\Sigma$. As one caveat, we remark that our algorithm does incur an additional factor of $\log(\lambda_n/\lambda_1)$ in the time complexity; we leave it as an interesting open problem whether this dependence can be removed.

\subsection{Lower bounds}\label{sec:related-lower}

There is a long line of work studying lower bounds on the sample complexity of \emph{specific} algorithms for sparse linear regression: see e.g. \cite{van2018tight}, which shows that the sample complexity of Lasso can be lower bounded in terms of the compatibility constant (in the fixed-design setting where $X^{(1)},\dots,X^{(m)}$ may be arbitrary), and precise high-dimensional asymptotics for the exact performance of the Lasso (e.g. \cite{bayati2011lasso,stojnic2013framework,amelunxen2014living}).
There are very simple covariate distributions where Lasso (with standard normalization) fails for some sparse signal: e.g., if two covariates are approximately equal, and the remaining $n-2$ covariates are independent, then even with zero noise, Lasso provably requires $\Omega(n)$ samples to learn the difference between the first two covariates (see the ``weak compatibility condition'' of \citep{kelner2021power}). Of course, in these simple examples, while the Lasso algorithm may fail, there is no inherent computational obstruction (e.g., detecting two correlated covariates is straightforward). 

Efforts to prove sample complexity lower bounds have been extended to parametric algorithm families like Lasso with coordinate-wise additively separable regularization \citep{zhang2017optimal}, and preconditioned Lasso programs \citep{kelner2021power, kelner2022distributional}. It can be shown that SLR is $\NP$-hard if the algorithm is forced to output a $k$-sparse estimate of the ground truth (note: when $\Sigma$ is rank degenerate, this can be a much stronger requirement than achieving zero prediction error), see e.g. \cite{natarajan1995sparse,zhang2014lower,gupte2020fine,gupte2021fine,foster2015variable} --- but hardness of improper learning probably cannot be based on $\NP$-hardness (see e.g. \cite{applebaum2008basing}).
Evidence has also been given that Gaussian SLR with $\Sigma = I_n$ exhibits a constant-factor gap between the information-theoretic and algorithmic sample complexities, see e.g. \cite{gamarnik2017sparse}. All together, the known lower bounds seem fundamentally different from the likely \emph{exponential} gap in random-design, general-covariance sparse linear regression.
Although our lower bound (\Cref{thm:no-subquadratic-alg-intro}) does not establish the anticipated exponential gap, it introduces the first polynomial gap for sparse linear regression under a reasonable computational assumption, and pins down the correct gap in our setting of latent variable models.

\paragraph{Sparse PCA.} In the classical, positive-spike, sparse PCA detection problem, we are given independent samples from either $N(0, I_n + \beta xx^\t)$ or $N(0, I_n)$ for a random $k$-sparse unit vector $x$ and some $\beta>0$. The goal is to distinguish between these two settings. Negative-spike sparse PCA is the variant where $\beta \in [-1,0)$; see \Cref{sec:overview-lower} for a formal definition. We write ``near-critical'' as an informal term to denote the regime where $\beta$ is very close to $-1$, because the problem is no longer well-defined when $\beta < -1$.

There is considerable evidence that any computationally efficient algorithm requires $\Omega(k^2)$ samples to solve the the positive-spike sparse PCA detection problem, with constant signal strength $\beta = \Theta(1)$ \citep{berthet2013complexity, krauthgamer2015semidefinite, ma2015sum,gao2017sparse,lu2018edge, brennan2019optimal, ding2023subexponential},
even though the information-theoretic limit is only $\Theta(k\log n)$ samples \citep{berthet2013optimal}. More generally, it's widely believed that computationally efficient algorithms with access to $m$ samples can only perform detection for signal strength $\beta = \Omega(\sqrt{k^2/m})$.

In the negative-spike setting, the same computational/statistical tradeoff is conjectured to hold (where signal strength is now measured by $-\beta$). But this has only been rigorously proven (under a variant of the Planted Clique conjecture) for sparsity $k = o(m^{1/6})$, or equivalently $-\beta = o(m^{-1/3})$ \citep{brennan2020reducibility}. 
On the one hand, there appear to be considerable technical challenges in proving reduction-based hardness of near-critical negative-spike sparse PCA \citep{brennan2020reducibility}. On the other hand, understanding the complexity of negative-spike sparse PCA seems to be an important problem in the theory of average-case hardness --- besides the new reduction in this work, previous work has connected the hardness of negative-spike sparse PCA to conjectured computational-statistical gaps in phase retrieval, mixed linear regression, and in certifying the RIP property \citep{brennan2020reducibility,ding2021average}. 
As discussed earlier, some known hardness results for \emph{positive-spike} sparse PCA in restricted classes of algorithms can be adapted to the negative-spike setting --- see \cite{ding2021average} for previous work away from the critical threshold, discussed further below, and our new results in \Cref{sec:ldlr} for low-degree and SDP hardness near the critical threshold.

\paragraph{Related problems: planted sparse vector and certifying RIP.} A long line of works have studied both upper and lower bounds for the problem of finding a planted sparse vector in a random subspace --- see \cite{ding2023sq,barak2014rounding,demanet2014scaling,mao2021optimal,hopkins2016fast,qu2014finding}. As we mentioned when we introduced Algorithm~\ref{alg:intro}, the convex program we solve in the inner loop of \textsc{SmartScaling} is similar to the convex relaxations used in  \cite{demanet2014scaling} and \cite{spielman2012exact} to search for sparse vectors in a subspace. Such a program is also similar in spirit to pseudolikelihood methods used for learning sparse graphical models, see e.g. \cite{besag1977efficiency,meinshausen2006high,kelner2020learning}. The idea of solving such a program iteratively seems new to this work. 

The most relevant lower bound in this line of work to us is the low-degree hardness result of \cite{ding2021average}. They phrased their lower bound as one for the problem of certifying the RIP property\footnote{As a reminder, the Restricted Isometry Property (RIP) is one that is sufficient, but not necessary, for methods such as the LASSO to achieve nearly-statistically optimal performance for sparse recovery. There are many randomized constructions of sensing matrices with good RIP properties, so that efficient sparse recovery/SLR is typically possible in such an ensemble. However, \emph{certification} that a particular sensing matrix is good may be computationally difficult and is a longstanding challenge. For the particular Gaussian ensemble considered by \cite{ding2021average}, the LASSO is statistically optimal up to constants (see e.g. \cite{raskutti2010restricted,candes2007dantzig}), but their lower bound gives evidence certification is computationally hard.} in the average case (c.f.\ \cite{wang2016average}), but as they discuss in their Remark II.4, their lower bound can also be interpreted as evidence for the optimality of the \cite{demanet2014scaling} linear program among computationally efficient algorithms, in the setting where the random subspace has small codimension.
Their hardness result for certifying RIP is established by proving a low-degree lower bound for a version of the negatively-spiked sparse PCA problem. The crucial difference between our setting and theirs is that they use a version of negatively-spiked sparse PCA where $\beta \in (-1,0)$ is a \emph{fixed constant} as the sparsity $k$ of the spike goes to infinity, whereas for us it is crucial to consider the ``near-critical'' regime where $\beta \to -1$ as $k$ goes to infinity. More precisely, our reduction to SLR operates in the regime $\beta < -1 + 1/\Omega(k)$; see the discussion around Remark~\ref{rmk:positive-pc-fails}. It is necessary that we take $\beta \to -1$, because when $\beta$ is a fixed constant, the SLR problems arising in our reduction are all well-conditioned, so they can actually be solved with nearly-optimal sample complexity in polynomial time using Lasso \citep{raskutti2010restricted} and cannot be the basis of a hardness result.
At a technical level, taking $\beta \to -1$ also leads to a difference in the low-degree analysis --- \cite{ding2021average} use that when $\beta$ is fixed, an i.i.d. prior for the spike will satisfy $\beta \|x\|^2 < 1$ with high probability, but in the near-critical regime this is not true 
(see Remark~\ref{rmk:previous-cannot-apply}).

%% file: ldlr.tex
In this section we prove \Cref{thm:no-subquadratic-alg-intro}, which asserts (under \Cref{conj:ldlr}, defined below) that no polynomial-time algorithm for sparse linear regression can achieve prediction error $\sigma^2/10$ with significantly less than $O(k^2 \log n)$ samples, even when $\Sigma$ is $(1,k)$-rescalable. Since \rl{} has sample complexity $O(k^2 \log n)$ in \Cref{setting:lvm} with up to $O(k^2)$ latent variables, it follows under the same conjecture that \rl{} is essentially optimal in that setting.

\paragraph{Section outline.} In \Cref{sec:ldlr-prelim}, we introduce \Cref{conj:ldlr} and other necessary background. In \Cref{sec:reduction}, we formally prove \Cref{theorem:pca-to-slr-overview}, which states that negative-spike $k$-sparse PCA can be reduced to $k$-sparse linear regression with $(1,k)$-rescalable covariance. In \Cref{sec:ldp}, we show that low-degree polynomials cannot solve near-critical negative spike sparse PCA with $o(k^2)$ samples, adapting an argument from \cite{bandeira2020computational}. Combined with \Cref{theorem:pca-to-slr-overview}, this yields \Cref{thm:no-subquadratic-alg-intro}. In \Cref{sec:sdp}, we give additional evidence of a statistical/computational tradeoff for negative spike sparse PCA, by showing that a natural semidefinite program also requires $\tilde\Omega(k^2)$ samples.


\begin{remark}\label{remark:analogy-spn}
The negative spiked sparse PCA testing problem can be thought of as a real-valued analogue of the foundational Sparse Parities with Noise (SPN) problem \citep{feldman2009agnostic}, where the task is to distinguish between $m$ independent samples from the null distribution $\Unif \{\pm 1\}^n$, and $m$ samples from a random planted distribution on $\{\pm 1\}^n$ defined as follows. First, a set $S \subseteq [n]$ of size $k$ is sampled uniformly at random. Then, the constraint $x_S = 0$ is noisily ``planted'' in each of the $m$ samples. Formally, conditioned on $S$, each of the $m$ samples is i.i.d. with distribution
\[ X \sim (1 - \epsilon) \Unif \left\{x : \sum_{i\in S} x_i = 0 \pmod{2} \right\} + \epsilon \Unif \left\{x : \sum_{i\in S} x_i = 1 \pmod{2} \right\}. \]
This problem is solvable via Gaussian elimination when $\epsilon = 0$ but is conjectured to require $n^{\Omega(k)}$ time when $\epsilon > 0$ is fixed.
See e.g. \cite{valiant2012finding} and references within for a history of this well-known problem. When $\epsilon$ is close to zero, we can informally view the planted measure as the result of conditioning on $\sum_{i\in S} X_i \approx 0$ (where the approximation is in the sense of $L^2$). 

The negative spike sparse PCA problem can similarly be viewed as testing between $N(0,I)$ and a planted distribution where for a random set $S$ (corresponding to the support of the spike vector $w$ in \Cref{def:pq}), we condition a standard Gaussian vector $X$ on the event that $\sum_{i\in S} X_i$ has smaller variance than in the null measure. In the near-critical case, we are essentially conditioning on $\sum_{i\in S} X_i \approx 0$. 

In both cases, the information-theoretic limit is at $\tilde O(k)$ samples. Of course, the computational limits are somewhat different, since SPN is believed to be computationally intractable even with $\poly(n)$ samples.
\end{remark}

\subsection{Preliminaries}\label{sec:ldlr-prelim}

Recall the definition of the sparse spike prior $\CW_{n,k}$ (\Cref{def:sparse-prior}), the sparse spiked Wishart distribution $\BP_{n,k,\beta,m}$, and the null distribution $\BQ_{n,m}$ (\Cref{def:pq}) from \Cref{sec:overview-lower}. We formally define the testing problem associated with these distributions, and the induced likelihood ratio.

\begin{definition}\label{def:strong-detection}
Fix functions $k, m: \NN \to \NN$ with $k(n) \leq n$, and $\beta: \NN \to (-1, \infty)$. An algorithm $\MA$ solves the \emph{strong detection} problem for the $k$-sparse spiked Wishart model (with parameter functions $k$, $m$, and $\beta$) if it distinguishes $\BP_{n,k(n),\beta(n),m(n)}$ from $\BQ_{n,m(n)}$ with probability $1-o(1)$, i.e.
\[\left|\Pr_{Z \sim \BP_{n,k(n),\beta(n),m(n)}}[\MA(Z)=1] -\Pr_{Z \sim \BQ_{n,m(n)}}[\MA(Z)=1]\right| = 1-o(1).\]
\end{definition}

\begin{definition}\label{def:l}
Let $n,k,m \in \NN$ with $k \leq n$ and $\beta \in (-1,\infty)$. We define the \emph{likelihood ratio} as $L_{n,k,\beta,m} := \frac{d\BP_{n,k,\beta,m}}{d\BQ_{n,m}}$.
\end{definition}

For a planted distribution $\BP$ and a null distribution $\BQ$, the (degree-$D$) \emph{low-degree likelihood ratio} (LDLR) $\norm{L^{\leq D}}_{L^2(\BQ)}$ is the norm under $L^2(\BQ)$ of the likelihood ratio $L = d\BP/d\BQ$ after orthogonal projection onto the space of multivariate polynomials of degree at most $D$. For a family of planted distributions $(\BP_n)_{n\in\NN}$ and null distributions $(\BQ_n)_{n\in\NN}$, if the low-degree likelihood ratio between $\BP_n$ and $\BQ_n$ can be bounded above by a constant as $n \to \infty$, then it can be seen that no degree-$D$ polynomial can distinguish $\BP_n$ from $\BQ_n$ with error $o(1)$ as $n \to \infty$ (see e.g. Proposition 1.15 in \cite{kunisky2019notes} and references). 

Moreover, it has been conjectured that for any ``natural'' statistical hypothesis testing problem, the best degree-$\log^{1+c} n$ polynomial (for any constant $c>0$) is at least as good a distinguisher as the best polynomial-time algorithm. Informally, this is known as the \emph{Low Degree Conjecture}. There are concrete, formal statements of this conjecture (see e.g. \cite{hopkins2018statistical}) for broad classes of statistical problems, although these do not specifically capture the spiked Wishart model. See also \cite{holmgren2020counterexamples,koehler2022reconstruction} for further discussion about the settings where low-degree polynomials are good proxies for polynomial time algorithms. Below we formalize the precise conjecture that needs to hold for our purposes (i.e. to prove non-existence of an efficient algorithm for negative-spike sparse PCA via a low-degree likelihood ratio bound).

\begin{conjecture}[Hardness thresholds for spiked Wishart match Low-Degree]\label{conj:ldlr}
Fix functions $k, m: \NN \to \NN$ with $k(n) \leq n$, and $\beta: \NN \to (-1, 0)$ with $1+\beta(n) \geq 1/\poly(n)$. If there exists some $D: \NN \to \NN$ with $D(n) = \log^{1+\Omega(1)} n$ and $\norm{L^{\leq D(n)}_{n,k(n),\beta(n),m(n)}}_{L^2(\BQ_{n,m(n)})} = O(1)$, then there is no randomized polynomial-time algorithm $\MA$ that solves strong detection (\Cref{def:strong-detection}) for the spiked Wishart model with parameter functions $k$, $m$, and $\beta$. 
\end{conjecture}
\begin{remark}
In this conjecture, we do not allow $\beta$ to equal $-1$ (or to converge to $-1$ more than polynomially fast). In part, this is because low-degree hardness is only expected to be a good heuristic when there is at least a small amount of noise in the underlying problem. If the underlying signal is binary valued and there is extremely little noise, algebraic methods like the LLL algorithm can sometimes be used to solve regression tasks with very few samples, see \cite{zadik2018high}.
\end{remark}

\subsection{Reduction from negative-spike sparse PCA to sparse linear regression}\label{sec:reduction}

We start by reducing (near-critical) negative-spike sparse PCA to sparse linear regression \--- with a covariance matrix that satisfies $(1,k)$-rescalability. As was explained in \Cref{sec:overview-lower}, the idea is to check whether any covariate in the given sparse PCA data can be explained by the other covariates better than one would expect under the null distribution. 

\begin{theorem}[Restatement of \Cref{theorem:pca-to-slr-overview}]\label{theorem:pca-to-slr}
Let $\mslr(n,k)$ be a function. Suppose that there is a polynomial-time sparse linear regression algorithm $\MA$ with the following property. For any $n,k\in\NN$, $\sigma>0$, positive semi-definite $(1,k)$-rescalable matrix $\Sigma \in \RR^{n\times n}$, $k$-sparse vector $w^\st \in \RR^n$, and $m \geq \mslr(n,k)$, the output $\wh \gets \MA((X^{(j)},y^{(j)})_{j=1}^m)$ satisfies
\[\Pr[\norm{\wh - w^\st}_\Sigma^2 \leq \sigma^2/10] \geq 1-o(1) \]
where the probability is over the randomness of $\MA$ and $m$ independent samples $(X^{(j)},y^{(j)})_{j=1}^m$ from $\SLR_{\Sigma,\sigma}(w^\st)$.

Then there is a polynomial-time algorithm $\MA'$ with the following property. For any $n,m,k \in \NN$ and $\beta \in (-1, -1 + 1/(2k)]$, if $m \geq \mslr(n,k) + 1600\log(n)$, then
\begin{equation} \left|\Pr_{Z\sim \BP_{n,k,\beta,m}}[\MA'(Z)=1] - \Pr_{Z\sim \BQ_{n,m}}[\MA'(Z)=1]\right| = 1 - o(1).\label{eq:detection-success} \end{equation}
\end{theorem}

To be precise, the asymptotics here (as elsewhere in this section) are in terms of $n$. That is, each term $o(1)$ represents a function of $n$ (that of course depends on the algorithm $\MA$, but not the other parameters) that goes to $0$ as $n \to \infty$.

\begin{proof}
The algorithm $\MA'$ on input $Z = (Z^{(j)})_{j=1}^m$ has the following behavior. Let $m' = m - 1600\log(n)$. For each $i \in [n]$, compute
\[\hat{w}^{(i)} := \MA((Z^{(j)}_{-i}, Z^{(j)}_i)_{j=1}^{m'})\]
and
\[\eta^{(i)} := \frac{1}{1600\log n} \sum_{j=m'+1}^m \left(Z^{(j)}_i - \langle Z^{(j)}_{-i}, \hat{w}^{(i)}\rangle\right)^2\]
The output of $\MA$ is then
\[\mathbbm{1}[\exists i \in [n]: \eta^{(i)} < 9/10].\]

\paragraph{Analysis.} It's clear that the time complexity of $\MA'$ is dominated by the time complexity of $\MA$ (multiplied by $n$), which is by assumption polynomial in $n$. It remains to check \Cref{eq:detection-success}. 

First, suppose that $Z \sim \BP_{n,k,\beta,m}$. Recall that $Z$ is sampled by first drawing a spike vector from $\CW_{n,k}$. Let us condition on this vector being some $w \in \RR^n$, which by definition of $\CW_{n,k}$ (\Cref{def:sparse-prior}) is $k$-sparse. Fix any $i \in \argmax_{j \in [n]} |w_j|$. We will show that $\eta^{(i)} < 9/10$ with high probability. By definition, the random variables $Z^{(1)},\dots,Z^{(m)}$ are independent and identically distributed (after conditioning on $w$). Fix any $j \in [m]$. Then $Z^{(j)} \sim N(0, I_n + \beta ww^\t)$. Thus, the marginal distribution of $Z^{(j)}_{-i}$ is $N(0, I_{n-1} + \beta w_{-i}w_{-i}^\t)$, and for any $x \in \RR^n$, $Z^{(j)}_i|Z^{(j)}_{-i}=x$ has distribution
\[N\left(\left\langle x, \frac{\beta w_i}{1+\beta(1-w_i^2)}w_{-i}\right\rangle, \frac{1+\beta}{1+\beta(1-w_i^2)}\right).\]
Define $\Sigma := I_{n-1} + \beta w_{-i}w_{-i}^\t$ and $\theta := \frac{\beta w_i}{1+\beta(1-w_i^2)}w_{-i}$ and $\sigma^2 := \frac{1+\beta}{1+\beta(1-w_i^2)}$. Then the tuple $(Z^{(j)}_{-i}, Z^{(j)}_i)$ is distributed according to $\SLR_{\Sigma,\sigma}(\theta)$. Since $w_{-i}$ is $(k-1)$-sparse, we get that $\theta$ is $(k-1)$-sparse.

Let $S := \supp(w_{-i}) \subseteq [n-1]$ and let $D \in \RR^{n-1 \times n-1}$ be the diagonal matrix defined by $D_{aa} := \mathbbm{1}[a \not \in \supp(w)]$ for $a \in [n-1]$. Then for any $v \in \RR^{n-1}$, \[\norm{v}_\Sigma^2 = \norm{v}_2^2 + \beta \langle v, w_{-i}\rangle^2 \geq \norm{v}_2^2 - \norm{v_{[n-1]\setminus S}}_2^2 = \norm{v}_D^2\]
where the inequality is by Cauchy-Schwarz, the assumption $\beta > -1$, and the bound $\norm{w_{-i}}_2 \leq \norm{w}_2 = 1$. Thus, $D \preceq \Sigma \preceq I_n$. Since $I_n - D$ has rank at most $k$, it follows that $\Sigma$ is $(1,k)$-rescalable.

We can now apply the theorem hypothesis. Since $m' \geq \mslr(n,k)$, it holds with probability $1-o(1)$ that $\hat w^{(i)} \gets \MA((Z^{(j)}_{-i},Z^{(j)}_i)_{i=1}^{m'})$ satisfies
\begin{equation} \norm{\hat w^{(i)} - \theta}_\Sigma^2 \leq \sigma^2/10 \leq 1/20.\label{eq:var-dec-bound}\end{equation}
where the last inequality uses that \[\sigma^2 = \frac{1+\beta}{1+\beta(1-w_i^2)} \leq \frac{1+\beta}{1+\beta-\beta/k} \leq \frac{1+\beta}{1+\beta+1/(2k)} \leq 1/2\] (since $|w_i| \geq 1/\sqrt{k}$ and $1+\beta \leq 1/(2k)$). Condition on the event that \Cref{eq:var-dec-bound} holds. Then for any $m' < j \leq m$, since $Z^{(j)}$ is independent of $\hat w^{(i)}$, we have
\[\EE \left(Z^{(j)}_i - \langle Z^{(j)}_{-i}, \hat{w}^{(i)}\rangle\right)^2 = \sigma^2 + \norm{\hat{w}^{(i)} - \theta}_\Sigma^2 \leq 3/5.\]
By concentration of $\chi^2$ random variables, it follows that $\eta^{(i)} < 9/10$ (and hence $\MA$ outputs $1$) with probability $1 - o(1)$ over $Z \sim \BP_{n,k,\beta,m}$.

On the other hand, suppose that $Z \sim \BQ_{n,m}$. Again, $(Z^{(j)})_{j=1}^m$ are independent. Fix $i \in [n]$ and condition on the first $m'$ samples, which fixes $\hat w^{(i)}$. For any $m' < j \leq m$, we know that $Z^{(j)}_i$ is independent of $\langle Z^{(j)}_{-i}, \hat w^{(i)}\rangle$, so
\[\EE \left(Z^{(j)}_i - \langle Z^{(j)}_{-i}, \hat{w}^{(i)}\rangle\right)^2 \geq \EE (Z^{(j)}_i)^2 = 1.\]
Concentration of $\chi^2$ random variables gives that $\eta^{(i)} \geq 9/10$ with probability at least $1 - 2n^{-2}$. A union bound over $i \in [n]$ implies that $\MA'$ outputs $0$ with probability $1 - o(1)$. This completes the proof of \Cref{eq:detection-success}.
\end{proof}

\subsection{Statistical efficiency of low-degree polynomials}\label{sec:ldp}

We next show that in the low-sample regime $m = O(k^2/D)$, the degree-$D$ likelihood ratio for the spiked Wishart model with parameter functions $k$, $\beta$, and $m$ (see \Cref{sec:ldlr-prelim}) is indeed bounded. Together with \Cref{theorem:pca-to-slr}, this proves a tight statistical/computational tradeoff for sparse linear regression with $(1,k)$-rescalable covariance, conditional on \Cref{conj:ldlr}.

Our starting point for bounding the low-degree likelihood ratio is the following instantiation of a general calculation due to \cite{bandeira2020computational}:

\begin{lemma}\label{lemma:ldlr-expr}
Let $n,k,m,D \in \NN$ with $k \leq n$ and $\beta \in (-1,\infty)$. Then
\[\norm{L^{\leq D}_{n,k,\beta,m}}_{L^2(\BQ_{n,m})}^2 = \E_{w_1, w_2 \sim \CW_{n,k}} \sum_{d=0}^{\lfloor D/2\rfloor} \left(\frac{\beta^2 \langle w_1, w_2\rangle}{4}\right)^d \sum_{\substack{d_1,\dots,d_m \\ \sum d_i = d}} \prod_{i=1}^m \binom{2d_i}{d_i}.\]
\end{lemma}

\begin{proof}
We apply Lemma 5.9 from \cite{bandeira2020computational}. We only need to check that $\CW_{n,k}$ is a $\beta$-good normalized spike prior (see Definitions~2.9 and~2.11 in \cite{bandeira2020computational}), but this is immediate from the fact that $\beta > -1$ and $\Pr_{w\sim \CW_{n,k}}[\norm{w}_2 \leq 1] = 1$.
\end{proof}

We now roughly follow the proof of \cite[Theorem 2.14(b)]{ding2023subexponential}, which gives a low-degree bound in the closely related setting where $\beta \geq 0$ and the sparse spike prior has independent entries. 

\begin{remark}[Inapplicability of previous bounds]\label{rmk:previous-cannot-apply}
Note that the above expression is an even function of $\beta$, so for any given well-defined spike prior, the low-degree computation in the negative-spike case is identical to that in the positive-spike case. Unfortunately, the spike prior with independent entries is only well-defined in the positive-spike case, since it's possible for the spike vector to have norm larger than one. If this degeneracy occurred only with vanishing probability, then one could hope to perform a truncation argument \citep{bandeira2020computational,ding2021average}, but in the near-critical regime that we care about (i.e. $\beta \in (-1, -1 + 1/(2k))$), a spike prior with i.i.d. entries and expected sparsity $k$ will have norm exceeding one with constant probability.

Thus, the fixed-size sparse spike prior $\CW_{n,k}$ (\Cref{def:sparse-prior}) seems crucial to the proof. We are not aware of a previous low-degree analysis with this prior, so we have to do it ourselves.
\end{remark}

We start by bounding the moments of $\langle w_1,w_2\rangle$ for independent $w_1,w_2\sim\CW_{n,k}$.

\begin{lemma}\label{lemma:ldlr-a-bound}
Let $n,k,d \in \NN$ with $k \leq \sqrt{n/(4e)}$. Then
\[A_{n,k,d} := k^{2d}\EE \langle w_1,w_2\rangle^{2d} \leq 2 \cdot (2d)^{2d}\]
where the expectation is over independent draws $w_1,w_2 \sim \CW_{n,k}$.
\end{lemma}
\begin{remark}
Note that the trivial bound (from Cauchy-Schwarz) is $A_{n,k,d} \leq k^{2d}$. However, for $d \ll k$ this is very loose. We improve it by using the fact that the supports of $w_1, w_2$ are unlikely to have large overlap.
\end{remark}
\begin{proof}
Define the (random) set \[S := \supp(w_1) \cap \supp(w_2) \subseteq [n].\] Let $\Rad(1/2)$ denote the Rademacher distribution $\Unif(\{-1,1\})$. Observe that after conditioning on any realization of $S$ with $|S| = \ell$, the random variable $\langle w_1,w_2\rangle$ has the distribution of $\frac{1}{k} \sum_{i=1}^\ell a_i$ where $a_1,\dots,a_\ell \sim \Rad(1/2)$ are independent. Thus,
\begin{align*}
A_{n,k,d}
&= \sum_{\ell=0}^k \Pr[|S|=\ell] \cdot \E_{a_1,\dots,a_\ell\sim\Rad(1/2)} \left(\sum_{i=1}^\ell a_i\right)^{2d} \\
&\leq \sum_{\ell=0}^k \Pr[|S| = \ell] \cdot \ell^{2d} \\ 
&\leq (2d)^{2d} + \sum_{\ell=2d+1}^k \Pr[|S|=\ell] \cdot \ell^{2d}.
\end{align*}
Define $g(\ell) := \Pr[|S| = \ell] \cdot \ell^{2d}$. Then for any $2d \leq \ell < k$, we have 
\begin{align*}
g(\ell+1)
&= \frac{\binom{k}{\ell+1}\binom{n-k}{k-\ell-1}}{\binom{n}{k}} (\ell+1)^{2d} \\ 
&= \frac{k-\ell}{\ell+1} \cdot \frac{k-\ell}{n-2k+\ell+1} \cdot \left(1 + \frac{1}{\ell}\right)^{2d} \cdot g(\ell) \\ 
&\leq \frac{2k^2}{n} e^{2d/\ell} g(\ell) \\ 
&\leq \frac{g(\ell)}{2}.
\end{align*}
Since $g(2d) \leq (2d)^{2d}$, it follows that
\[A_{n,k,d} \leq (2d)^{2d} + \sum_{\ell=2d+1}^k (2d)^{2d} (1/2)^{\ell-2d} \leq 2 \cdot (2d)^{2d}\]
as claimed.
\end{proof}

We also use the following bound from \cite{ding2023subexponential}:

\begin{lemma}[Lemma 4.7 in \cite{ding2023subexponential}]\label{lemma:ldlr-coef-bound}
There are constants $c_1,c_2>0$ with the following property. Let $m, D: \NN \to \NN$ be functions with $D = o(m)$. Then for all sufficiently large $n \in \NN$, for all $1 \leq d \leq D(n)$, it holds that
\[\sum_{\substack{d_1,\dots,d_m\geq 0 \\ \sum d_i = d}} \prod_{i=1}^{m(n)} \binom{2d_i}{d_i} \leq c_1 d^{3/2} e^{c_2 d^2/m(n)} \frac{(2m(n))^d}{d!}.\]
\end{lemma}

Combining the above pieces, we get the following bound on the low-degree likelihood ratio.

\begin{theorem}\label{theorem:ldlr-main}
Let $k, m, D: \NN \to \NN$ be functions with $2e\sqrt{m(n)D(n)} \leq k(n) \leq \sqrt{n/(4e)}$ for sufficiently large $n \in \NN$, and $D(n) = o(m(n))$. Let $\beta: \NN \to (-1,1)$. Then $\norm{L^{\leq D(n)}_{n,k(n),\beta(n),m(n)}}_{L^2(\BQ_{n,m(n)})} \leq O(1)$.
\end{theorem}
\allowdisplaybreaks

\begin{proof}
For notational simplicity, we write $k=k(n)$, $D = D(n)$, and $m = m(n)$. Applying \Cref{lemma:ldlr-expr} and the definition of $A_{n,k,d}$ (\Cref{lemma:ldlr-a-bound}), we have that for all sufficiently large $n$,
\begin{align*}
\norm{L^{\leq D}_{n,k,\beta,m}}_{L^2(\BQ_{n,m})}^2 
&= 1+ \sum_{d=0}^{\lfloor D/2\rfloor} \left(\frac{\beta^2}{4k^2}\right)^d A_{n,k,d}\sum_{\substack{d_1,\dots,d_m \\ \sum d_i = d}} \prod_{i=1}^m \binom{2d_i}{d_i} \\ 
&\leq 1 + 2c_1 \sum_{d=1}^{\lfloor D/2\rfloor} \left(\frac{\beta^2}{4k^2}\right)^d (2d)^{2d} d^{3/2} e^{c_2 d^2/m} \frac{(2m)^d}{d!} \\ 
&\leq 1 + 2c_1 \sum_{d=1}^{\lfloor D/2\rfloor} d^{3/2} \left(\frac{8ed^2 m}{4k^2 (d!)^{1/d}}\right)^d \\ 
&\leq 1 + 2c_1 \sum_{d=1}^{\lfloor D/2\rfloor} d^{3/2} \left(\frac{8e^2 d m}{4k^2}\right)^d \\ 
&\leq 1 + 2c_1 \sum_{d=1}^{\lfloor D/2\rfloor} d^{3/2} \left(1/2\right)^d
\end{align*}
where the first inequality applies \Cref{lemma:ldlr-a-bound} (using that $k \leq \sqrt{n/(4e)}$) and \Cref{lemma:ldlr-coef-bound} (using that $D = o(m)$); the second inequality uses the bounds $\beta^2 \leq 1$ and $e^{c_2 d^2/m} \leq e^{c_2 dD/m} \leq e^d$ (which holds for sufficiently large $n$, since $D = o(m)$); the third inequality uses Stirling's approximation; and the fourth inequality uses the assumption that $k^2 \geq 2e^2 mD$. It's clear that the final summation is upper bounded by an absolute constant, which completes the proof.
\end{proof}

\begin{corollary}[Restatement of \Cref{thm:no-subquadratic-alg-intro}]\label{cor:no-subquadratic-alg}
Let $\epsilon,C>0$ with $\epsilon \leq 2$. Suppose that there is a polynomial-time algorithm $\MA$ satisfying the following property. For any $n,k \in \NN$, $\sigma>0$, positive semi-definite, $(1,k)$-rescalable matrix $\Sigma \in \RR^{n\times n}$, $k$-sparse vector $w^\st \in \RR^n$, and $m \geq C k^{2-\epsilon}\log n$, the output $\hat w \gets \MA((X^{(j)},y^{(j)})_{j=1}^m)$ satisfies
\[\Pr[\norm{\hat w - w^\st}_\Sigma^2 \leq \sigma^2/10] \geq 1-o(1)\]
where the probability is over the randomness of $\MA$ and $m$ independent samples $(X^{(j)},y^{(j)})_{j=1}^m$ from $\SLR_{\Sigma,\sigma}(w^\st)$. Then \Cref{conj:ldlr} is false.
\end{corollary}

\begin{proof}
Define functions $m, k,D: \NN \to \NN$ and $\beta: \NN \to (-1,0)$ by $k(n) := \log^{10/\epsilon} n$, $m(n) := Ck(n)^{2-\epsilon}\log^3 n + 1600\log(n)$, $D(n) := (2e)^{-2}\log^2 n$, and $\beta(n) = -1 + 1/(2k(n))$. By \Cref{theorem:pca-to-slr}, there is a polynomial-time algorithm $\MA'$ that solves strong detection for the spiked Wishart model with parameter functions $k$, $m$, and $\beta$. But now
\[2e\sqrt{m(n)D(n)} \leq \sqrt{(C+1600)k(n)^{2-\epsilon}\log^5 n} \leq \sqrt{\frac{(C+1600)k(n)^2}{\log^5 n}} \leq k(n)\]
for sufficiently large $n \in \NN$. Also, $k(n) = \log^{10/\epsilon} n \leq \sqrt{n/(4e)}$ for sufficiently large $n$. Finally, $m(n) = \Omega(\log^3 n)$, so $D(n) = o(m(n))$. We can therefore apply \Cref{theorem:ldlr-main}, which gives that $\norm{L^{\leq D(n)}_{n,k(n),\beta(n),m(n)}}_{L^2(\BQ_{n,m(n)})} \leq O(1)$. Together with the guarantee on $\MA'$, this contradicts \Cref{conj:ldlr}.
\end{proof}

\begin{remark}\label{remark:ggms}
\Cref{theorem:ldlr-main} also implies a computational-statistical gap for learning Gaussian Graphical Models, under \Cref{conj:ldlr}. In particular, for any $\beta \in (-1, -1 + 1/k)$ and any $k$-sparse $w \in \RR^n$, define $\Sigma := I_n + \beta ww^\t$. Suppose that $w$ lies in the support of $\CW_{n,k}$, so that every nonzero entry of $w$ lies in $\{-1/\sqrt{k}, 1/\sqrt{k}\}$. Then it can be checked that the Gaussian Graphical Model with distribution $N(0,\Sigma)$ is $\kappa$-nondegenerate for some $\kappa = \Omega(1)$, i.e. the precision matrix $\Theta := \Sigma^{-1}$ satisfies $|\Theta_{ij}| \geq \Theta(1) \cdot \sqrt{\Theta_{ii}\Theta_{jj}}$ for all $i,j \in [n]$ with $\Theta_{ij} \neq 0$.

For any $\kappa$-nondegenerate Gaussian Graphical Model $N(0,\Sigma)$ with maximum degree $k$, the Markov structure (i.e. the support of $\Sigma^{-1}$) can be information-theoretically learned with $O(k\log(n)/\kappa^2)$ samples \citep{misra2017information}. But for any $\epsilon>0$, if there were a computationally efficient algorithm for learning such models with $O(k^{2-\epsilon}\log(n)/\kappa^2)$ samples, then by taking $\kappa$ to be a constant, there would be a computationally efficient algorithm for distinguishing the spiked Wishart distribution $\BP_{n,k,\beta,m}$ from the null distribution $\BQ_{n,m}$ for some $m = O(k^{2-\epsilon}\log(n))$, so long as $\beta \in (-1,-1+1/k)$: simply learn the underlying structure and check if it is the empty graph or not. But such an algorithm would contradict \Cref{theorem:ldlr-main}, assuming that \Cref{conj:ldlr} holds.

Finally, we observe that this lower bound actually holds for the natural testing problem of distinguishing between an empty GGM and a sparse GGM with at least one nondegenerate edge. We establish upper bound with matching dependence on $k$ later in Section~\ref{sec:testing}.
\end{remark}

\subsection{Statistical efficiency of semi-definite programming}\label{sec:sdp}

We now give a second piece of evidence that negative-spike sparse PCA may exhibits a statistical/computational tradeoff, based on the failure of a natural semi-definite program for solving the detection problem. Given samples $Z^{(1)},\dots,Z^{(m)} \in \RR^n$, one can solve the following program in polynomial time:
\begin{equation}
V_k(\hat\Sigma) := \min_{A \in \RR^{n\times n}: A \succeq 0} \langle \hat \Sigma, A\rangle \quad\text{s.t. } \tr(A) = 1 \text{ and } \sum_{i,j=1}^n |A_{ij}| \leq k
\label{eq:sdp}
\end{equation}
where $\hat\Sigma = \frac{1}{m} \sum_{j=1}^m Z^{(j)} (Z^{(j)})^\t$ is the empirical covariance of the samples $(Z^{(j)})_{j=1}^m$. This is the same as the standard semi-definite programming relaxation suggested for positive spike sparse PCA \cite{d2004direct, krauthgamer2013semidefinite} except that the program minimizes rather than maximizing.\footnote{This difference means that we need a different construction than they use to establish the SDP lower bound --- ours is based on looking at the kernel of the empirical covariance.}

Since the matrix $A := ww^\t$ is feasible for \eqref{eq:sdp}, the value of the program is at most $O(1+\beta)$ with high probability over $(X^{(j)})_{j=1}^m \sim \BP_{n,k,\beta,m}$ (so long as $m = \Omega(k\log n)$). One would hope that under the null hypothesis $(Z^{(j)})_{j=1}^m \sim \BQ_{n,m}$, the program value $V_k(\hat\Sigma)$ is with high probability concentrated near one, in which case \eqref{eq:sdp} would solve the strong detection problem. However, we show that this is not the case when $m \ll k^2$: in this regime, the value of the program is in fact \emph{zero} with high probability under the null hypothesis (\Cref{thm:sdp}). Since the value is always nonnegative, it follows that the natural, computationally efficient test based on this program \--- reject the null hypothesis if $V_k(\hat\Sigma)$ is below some threshold \--- fails to solve the strong detection problem if the number of samples $m$ is significantly less than $O(k^2)$.

We prove \Cref{thm:sdp} by taking $A$ in the above program to be (an appropriate rescaling of) the orthogonal projection onto $\ker(\hat \Sigma)$. It's clear that $\langle \hat \Sigma, A\rangle = 0$, so it remains to check that $A$ is feasible for \eqref{eq:sdp} with high probability, when rescaled to have trace one.

Quantitatively, if $Q$ is an orthogonal projection matrix onto a random $(n-m)$-dimensional subspace of $\RR^n$ (as $\vspan(\hat \Sigma)$ indeed is), we would like to bound $\sum_{i,j = 1}^n |Q_{ij}|$ by roughly $n\sqrt{m}$. Since $I_n - Q$ has the same sum of entries up to $O(n)$, we can equivalently consider an orthogonal projection matrix $P$ onto a random $m$-dimensional subspace. The following lemma gives the desired bound on the sum of entries of $P$ (essentially as an application of Johnson-Lindenstrauss concentration).

\begin{lemma}\label{lemma:proj-norm}
Let $n,m \in \NN$ with $m \leq n$, and let $V \subseteq \RR^n$ be a uniformly random $m$-dimensional subspace. Let $P \in \RR^{n\times n}$ be the orthogonal projection matrix onto $V$. Then for any $\delta>0$,
\[\Pr\left[\sum_{i,j=1}^n |P_{ij}| > 11n\sqrt{m\log(n/\delta)}\right] \leq \delta.\]
\end{lemma}

\begin{proof}
For any fixed unit vector $v \in \RR^n$, it is known \cite[Lemma 2.2]{dasgupta2003elementary} that
\[\Pr\left[\left|\norm{v}_P^2 - \frac{m}{n}\right| > \frac{\epsilon m}{n}\right] \leq \exp(-k\epsilon^2/12).\]
Define $S := \{(e_i+e_j)/\sqrt{2}: i,j \in [n] \land i\neq j\} \cup \{e_i: i \in [n]\}$ and $\epsilon := \sqrt{12\log(n^2/\delta)/m}$. By a union bound, we get that $|\norm{v}_P^2 - m/n| \leq \epsilon m/n$ for all $v \in S$, with probability at least $1-\delta$. Henceforth we condition on this event.

First, $\sum_{i=1}^n |P_{ii}| = \tr(P) = m$. Next, fix any $i,j \in [n]$ with $i \neq j$. Then
\[P_{ij} = e_i^\t P e_j = \frac{1}{2}\left(\norm{e_i+e_j}_P^2 - \norm{e_i}_P^2 - \norm{e_j}_P^2\right).\]
Thus,
\[|P_{ij}| \leq \frac{1}{2}\left|\norm{e_i+e_j}_P^2 - \frac{2m}{n}\right| + \frac{1}{2}\left|\norm{e_i}_P^2 - \frac{m}{n}\right| + \frac{1}{2}\left|\norm{e_j}_P^2 - \frac{m}{n}\right| \leq \frac{2\epsilon m}{n}.\]
Substituting in the choice of $\epsilon$, we get that
\[\sum_{i,j=1}^n |P_{ij}| \leq n + 2\epsilon m n \leq n + 2n\sqrt{12m\log(n^2/\delta)} \leq 11n\sqrt{m\log(n/\delta)}\]
as claimed.
\end{proof}

We can now prove the claimed result that \eqref{eq:sdp} has value zero with high probability under the null hypothesis.

\begin{theorem}\label{thm:sdp}
There is a constant $c>0$ with the following property. Let $n,m,k \in \NN$, and let $(Z^{(j)})_{j=1}^m \sim \BQ_{n,m}$. Let $\hat\Sigma := \frac{1}{m}\sum_{j=1}^m Z^{(j)} (Z^{(j)})^\t$. For any $\delta>0$, if $m \leq \min(ck^2/\log(n/\delta), n/2)$, then 
\[\Pr[V_k(\hat\Sigma) > 0] \leq \delta.\]
\end{theorem}

\begin{proof}
Let $P \in \RR^{n\times n}$ be the orthogonal projection onto $\vspan(\hat\Sigma)$ and let $A = \frac{I_n - P}{n-m}$. Note that $\tr(P) = \vrank(\hat\Sigma) = m$ almost surely, so $\tr(A) = 1$ almost surely. Moreover $I_n - P = \sum_{i=1}^{n-m} w_i w_i^\t$ where $w_1,\dots,w_{n-m}$ form an orthonormal basis for $\ker(\hat \Sigma)$. Hence, $\langle \hat\Sigma, A\rangle = \frac{1}{n-m}\sum_{i=1}^{n-m} w_i^\t \hat\Sigma w_i  = 0$. Thus, if $\sum_{i,j=1}^n |A_{ij}| \leq k$, then $V_k(\hat\Sigma) = 0$. It only remains to bound the probability that $\sum_{i,j=1}^n |A_{ij}| > k$. By \Cref{lemma:proj-norm} and the fact that $\vspan(\hat\Sigma)$ is a uniformly random $m$-dimensional subspace of $\RR^n$, with probability at least $1-\delta$ we have $\sum_{i,j=1}^n |P_{ij}| \leq 11n\sqrt{m\log(n/\delta)}$. In this event,
\begin{align*}
\sum_{i,j=1}^n |A_{ij}| 
&\leq \frac{n}{n-m} + \frac{1}{n-m}\sum_{i,j=1}^n |P_{ij}| \\ 
&\leq \frac{12n\sqrt{m\log(n/\delta)}}{n-m} \\ 
&\leq 24\sqrt{m\log(n/\delta)} \\
&\leq k
\end{align*}
where the third inequality uses the assumption that $m \leq n/2$, and the fourth inequality holds so long as $c \geq 576$.
\end{proof}